\documentclass[10pt]{asme2ej}

\newif\ifPrePrint
\PrePrinttrue

\title{Static Hovering Realization for Multirotor Aerial Vehicles with Tiltable Propellers}

\author{Mahmoud Hamandi\thanks{Corresponding author.}
    \affiliation{ 
	\small Post-Doctoral Fellow\\
	\small Department of Aerospace Engineering, \\
	\small Khalifa University Center for\\
	\small Autonomous Robotic Systems \\(KUCARS),\\
	\small Khalifa University, \\
	\small Abu Dhabi, United Arab Emirates.\\
    \small Email: mahmoud.hamandi@ku.ac.ae
    }	
}

\author{Lakmal~Seneviratne
    \affiliation{
	\small Professor\\
	\small Department of Mechanical Engineering, \\
	\small and Director of
	\small Khalifa University Center for \\
	\small Autonomous Robotic Systems \\(KUCARS),\\
	\small Khalifa University, \\
	\small Abu Dhabi, United Arab Emirates.\\
    \small Email: lakmal.seneviratne@ku.ac.ae
    }	
}

\author{Yahya~Zweiri
    \affiliation{
	\small Associate Professor\\
	\small Department of Aerospace Engineering, \\
	\small Khalifa University Center for Autonomous Robotic Systems (KUCARS),\\
	\small Khalifa University, \\
	\small Abu Dhabi, United Arab Emirates.\\
    \small Email: yahya.zweiri@ku.ac.ae
    }	
}

\usepackage{amsmath, amssymb,amsfonts}
\usepackage{amsthm}
\usepackage{graphicx}
\usepackage{textcomp}
\def\BibTeX{{\rm B\kern-.05em{\sc i\kern-.025em b}\kern-.08em
    T\kern-.1667em\lower.7ex\hbox{E}\kern-.125emX}}

\usepackage{mathtools}

\usepackage{bm}
\providecommand{\bm}{\pmb}

\newtheorem{prop}{Proposition}

\newtheorem{thm}{Theorem}
\newtheorem{lem}[thm]{Lemma}

\theoremstyle{definition}
\newtheorem{dfn}{Definition}

\theoremstyle{remark}
\newtheorem{rmk}{Remark}

\newtheorem{ass}{Assumption}
\newtheorem{property}{Property}

\usepackage{color}
\usepackage{caption}

\usepackage{mathptmx}
\usepackage{times}
\usepackage{subcaption}

\usepackage{tikz}
\usepackage{placeins}
\usepackage{caption}
\usepackage[bookmarks]{hyperref} 
\hypersetup{
    colorlinks,
    citecolor=black,
    filecolor=black,
    linkcolor=black,
    urlcolor=black,
    pdfauthor={},
    pdfsubject={},
    pdftitle={}
}

\graphicspath{{Figures/}}

\usepackage{graphicx}

\usepackage{latexsym}
\usepackage{color}
\usepackage{paralist}

\definecolor{myred}{rgb}{0.8500, 0.3250, 0.0980}
\definecolor{myblue}{rgb}{0    0.4470    0.7410}
\definecolor{mygreen}{rgb}{0.4660, 0.6740, 0.1880}
\definecolor{mypurple}{RGB}{139,0,139}

\usepackage{color}
\newcommand\red[1]{{\textcolor{red}{#1}}}

\usepackage[normalem]{ulem}

\newcommand{\rev}[1]{{#1}}

\newcommand{\fig}{Fig.~}	
\newcommand{\eqn}{Eq.~}	
\newcommand{\eqns}{Eqs.~}	

\newcommand{\eqrefs}[2]{(\ref{#1}), (\ref{#2})}
\newcommand{\wrt}{w.r.t.}
\newcommand{\ie}{\textit{i.e.},~}

\usepackage[textfont=small]{caption} 
\DeclareMathOperator*{\rank}{rank}

\DeclareMathOperator*{\interior}{int}
\DeclareMathOperator*{\argmax}{arg\,max}

\DeclareMathOperator*{\Image}{Im}
\DeclarePairedDelimiter{\norm}{\lVert}{\rVert} 
\newcommand\at[2]{\left.#1\right|_{#2}}

\newcommand{\vect}[1]{\bm{#1}}		
\newcommand{\matr}[1]{\bm{#1}}		
\newcommand{\nR}[1]{\mathbb{R}^{#1}}		
\newcommand{\nN}[1]{\mathbb{N}^{#1}}		
\newcommand{\upperRomannumeral}[1]{\uppercase\expandafter{\romannumeral#1}}	

\newcommand{\Fm}{\vect{F}}

\newcommand{\mv}{\vect{m}}

\newcommand{\wv}{\vect{w}}

\bmdefine{\bg}{g}
\bmdefine{\bomega}{\omega}
\bmdefine{\bOmega}{\Omega}
\bmdefine{\boldetha}{\eta}

\newcommand{\frameV}{\mathcal{F}}		
\newcommand{\origin}{O}						
\newcommand{\vX}{\vect{x}}					
\newcommand{\vY}{\vect{y}}					
\newcommand{\vZ}{\vect{z}}					
\newcommand{\pos}{\vect{p}}				
\newcommand{\dpos}{{\vect{v}}}		
\newcommand{\ddpos}{\dot{\vect{v}}}	
\newcommand{\rotMat}{\matr{R}}			
\newcommand{\angVel}{\vect{\omega}}				
\newcommand{\angAcc}{\dot{\vect{\omega}}}		
\renewcommand{\skew}[1]{\matr{S}(#1)}				
\newcommand{\eye}[1]{\matr{I}_{#1}}		

\newcommand{\frameW}{\frameV_W}			
\newcommand{\originW}{\origin_W}		
\newcommand{\xW}{\vX_W}				
\newcommand{\yW}{\vY_W}				
\newcommand{\zW}{\vZ_W}				

\newcommand{\massB}{m}
\newcommand{\inertiaB}{\matr{J}_B}
\newcommand{\frameB}{\frameV_{B}}			
\newcommand{\originB}{O_{B}}					
\newcommand{\xB}{\vX_{B}}								
\newcommand{\yB}{\vY_{B}}								
\newcommand{\zB}{\vZ_{B}}								
\newcommand{\pB}{\pos_{B}}						
\newcommand{\dpB}{\dpos_B} 
\newcommand{\ddpB}{\ddpos_{B}}				
\newcommand{\angVelB}{\angVel_B}		
\newcommand{\angAccB}{\angAcc_B}		

\newcommand{\fB}{\vect{f}_B}       
\newcommand{\mB}{\vect{m}_B}       

\newcommand{\pPB}[1]{{}^{B}\pos_{#1}}						
\newcommand{\vP}[1]{{}^{B}\vect{v}_{#1}}						
\newcommand{\VP}{\vect{V}}						
\newcommand{\vO}[1]{{}^{B}\bar{\vect{v}}_{#1}}						
\newcommand{\uP}[1]{u_{#1}}						
\newcommand{\uPDer}[1]{\dot{u}_{#1}}
\newcommand{\ai}[1]{\alpha_{#1}}
\newcommand{\betai}[1]{\beta_{#1}}
\newcommand{\aid}[1]{\dot{\alpha}_{#1}}
\newcommand{\bid}[1]{\dot{\beta}_{#1}}
\newcommand{\thrusts}{\vect{u}} 
\newcommand{\thrustsDer}{\dot{\vect{u}}} 
\newcommand{\orientP}{\vect{\Omega}}  
\newcommand{\orientPDer}{\dot{\vect{\Omega}}}  
\newcommand{\thrustsD}{\dot{\vect{u}}} 
\newcommand{\orientPD}{\dot{\vect{\Omega}}}  
\newcommand{\wrench}{\vect{w}} %
\newcommand{\wrenchD}{\dot{\vect{w}}}

\newcommand{\thrustSet}{\mathcal{U}}
\newcommand{\thrustSetDer}{\dot{\mathcal{U}}}
\newcommand{\orientSet}{\mathcal{O}}
\newcommand{\orientSetDer}{\dot{\mathcal{O}}}
\newcommand{\cntrlDer}{\dot{\mathcal{C}}}
\newcommand{\cntrlSet}{\mathcal{C}}
\newcommand{\wrenchSet}{\mathcal{W}}

\newcommand{\vSet}{\mathcal{V}}
\newcommand{\momentSet}{\mathcal{M}}
\newcommand{\momentSetLocal}{\mathcal{M}_0}
\newcommand{\forceSet}{\mathcal{F}}

\newcommand{\pBd}{\pos_{B}^d}

\newcommand{\rotMatB}{\rotMat_{B}}			

\newcommand{\totalMoment}{\mv}
\newcommand{\totalWrench}{\wv}

\newcommand{\allocationMatrix}{\Fm}

\ifPrePrint
	\makeatletter
	\def\ps@titlepagestyle{
		\def\@oddfoot{}\def\@evenfoot{}
		\def\@oddhead{\textcolor{red}{\sf\footnotesize Preprint version \hfill}}
		\def\@evenhead{\textcolor{red}{\sf\footnotesize  Preprint version \hfill}}%
	}%
	\makeatother
	\makeatletter
	\def\ps@headings{
		\def\@oddfoot{\textcolor{red}{\sf\footnotesize  Preprint version \hfill}}\def\@evenfoot{\hfill\thepage\hfill}
		\def\@oddhead{}\def\@evenhead{}%
	}%
	\makeatother
\fi	
\begin{document}

\maketitle    

\begin{abstract}
{\it 
This paper presents a theoretical study on the ability of multi-rotor aerial vehicles (MRAVs) with tiltable propellers to achieve and sustain static hovering at different orientations.
To analyze the ability of MRAVs with tiltable propellers to achieve static hovering, a novel linear map between the platform's control inputs and applied forces and moments is introduced. The relation between the introduced map and the platform's ability to hover at different orientations is developed. Correspondingly, the conditions for MRAVs with tiltable propellers to realize and sustain static hovering are detailed.
A numerical metric is then introduced, which reflects the ability of MRAVs to sustain static hovering at different orientations.
A subclass of MRAVs with tiltable propellers is defined as the Critically Statically Hoverable platforms (CSH), where CSH platforms 
are MRAVs that cannot sustain static hovering with fixed propellers, but can achieve static hovering with tilting propellers.
Finally, extensive simulations are conducted to test and validate the above findings, and to demonstrate the effect of the proposed numerical metric on the platform's dynamics.}
\end{abstract}



\maketitle

\section{Introduction}

Multi-Rotor Aerial Vehicles (MRAVs) have been widely studied in the literature in the last decade. These platforms have seen an increase in their actuation abilities from the first presented designs \cite{pounds2002design} demonstrating the ability to achieve static hovering, to modern platforms able to apply independent forces and moments \cite{rashad2020ras} in all directions, or achieve omnidirectinal flight \cite{brescianini2016design,hamandi2020icuas}.
This increase in abilities allowed new applications to emerge that were previously not possible with fixed wing Unmanned Aerial Vehicles (UAV)s \cite{ducard2021review}, most notably, Aerial Physical Interaction \cite{ollero2022tro,mohiuddin2020energy} and Human Robot Interaction~\cite{allenspach2022ral}.

The majority of MRAV solutions rely on designs with fixed propellers (\ie the direction of thrust generated by each propeller is constant with respect to the body frame). Platforms with fixed propellers require less moving parts, and are thus easier to design, build \rev{and maintain}. However, and as shown in \cite{hamandi2021ijrr}, platforms with tiltable propellers could achieve desired actuation properties with fewer propellers. Therefore, they provide an interesting solution to many UAV applications \cite{robotics7040065}. In addition, since these platforms can achieve desired actuation properties with fewer propellers, they could withstand propeller failures, and maintain hover after the failure of one or more propellers.

In this perspective, there has been multiple prototype platforms in the literature that demonstrate the actuation abilities of MRAVs with tiltable propellers, with two \cite{cardoso2016robust}, three \cite{papachristos2016dual}, four \cite{alali2020jmes} and more propellers \cite{kamel2018voliro}. However, it is noted that each presented prototype proposes different strategies for actuation allocation, \ie  for mapping and allocating the platform's controls to achieve the desired forces and moments. 
For example, the authors in \cite{gress2002using} study the effect of changes in propellers' \rev{orientations} and \rev{thrusts}
on each force or moment axis, while operating near hover. 
While such a method allowed earlier platforms with actuated propellers to achieve static hovering and follow desired trajectories, they do not easily generalize to different platforms.
On the other hand, the authors in \cite{ramp2015modeling} propose a direct allocation method, where a linear map is proposed between the platform's wrench (\ie forces and moments) and the three-dimensional thrust generated by each propeller.
The authors in \cite{ryll2015novel} propose to control platforms with tiltable propellers at a higher differential level, where they propose a linear map between the wrench derivative and the derivative of the propellers' thrusts and orientations. It is noted that the linear map in \cite{ryll2015novel} is variable, as it depends on the current control inputs, while the one from \cite{ramp2015modeling} is fixed. On the other hand, the map from \cite{ryll2015novel} can be derived for any MRAV with tiltable propellers, while the one from~\cite{ramp2015modeling} can only encompass platforms with propellers tilting in $\nR{3}$.

\rev{Multiple} papers \rev{from the literature} analyzed thoroughly the ability of platforms with fixed propellers to hover at different orientations~\cite{michieletto2017}, achieve omnidirectional flight\cite{hamandi2021ijrr}, \rev{and hover following the failure of some of their propellers \cite{baskaya2021ral,remma2018cdc,shunsuke2021tro}.}
\rev{The authors in \cite{doi:10.1177/0278364920943654} analyze the hovering efficiency of different platforms, and show the ability of their proposed tiltrotor platform to achieve omnidirectional flight, with a higher efficiency than other fixed omnidirectional platforms from the literature.}
\rev{While the advantage of MRAVs with tilting propellers is clear from the corresponding designs in the literature (example \cite{soto2018hexapodopter} and \cite{mcarthur2018design}), the analysis of the ability of generic MRAV with tiltable propellers to achieve static hovering, and sustain static hovering is still missing in the literature.}
The work in \cite{hamandi2021ijrr} and \cite{hamandi2021airpharo} discuss such abilities for both fixed and actuated propellers\rev{; however, }their formulation is generic to both types of platforms, and as such, does not cover comprehensively the effect of propeller tilting on the platform's ability to hover.

This paper addresses this gap, and presents a comprehensive formulation that reflects the static hovering ability of any MRAV with tiltable propellers at different orientations. Then at each possible orientation, we analyze the stability of the corresponding hover, \ie  the ability to sustain achieved hovering in the presence of external disturbances. The main novel contributions of this paper are as follows:
\begin{enumerate}
\item Introduce a generic linear map between the wrench and controls for any tiltable MRAV that reflects all hovering orientations.
\item Analyze the ability of tiltable MRAVs to achieve static hovering at different orientations.
\item Analyze the ability of  tiltable MRAVs to sustain static hovering at different orientations, \rev{\ie the ability of tiltable MRAV to sustain static hovering while tilting between different propeller configurations.} Introduce a numerical metric that reflects the \rev{corresponding} ability.
\item Introduce a subclass of MRAVs with tiltable propellers
that cannot sustain static hovering with fixed propellers, but achieve static hovering with tilting propellers. We refer to platforms in this class as Critically Statically Hoverable (CSH) platforms.
\item Demonstrate the static hovering ability and fail-safe robustness of the proposed CSH platforms in extensive simulation.
\item Demonstrate \rev{in simulation} the effect of the proposed numerical metric on the platform's dynamic response at different hovering orientations.
\end{enumerate}

The rest of this paper is organized as follows: Section \ref{sec:modeling} summarizes the required definitions for the discussion of this paper, and introduces the allocation map for generic MRAVs with tiltable propellers. Section \ref{sec:static_hovering} discusses static hovering of generic MRAVs and the corresponding actuation properties. Section \ref{sec:csh} introduces Critically Statically Hoverable platforms, and demonstrates with a few examples how such platforms can achieve different actuation properties. Section \ref{sec:case_study} tests and validates the theoretical findings with extensive simulations, and finally, Section \ref{sec:conc} concludes the paper.


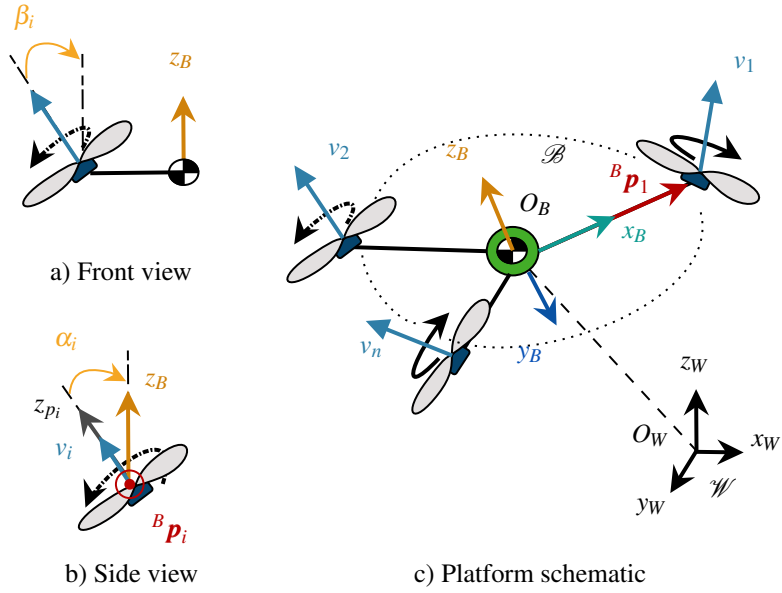
\begin{figure}
\centering

\tikzset{every picture/.style={line width=0.75pt}} 

\begin{tikzpicture}[x=0.75pt,y=0.75pt,yscale=-1,xscale=1]

\draw [color={rgb, 255:red, 0; green, 0; blue, 0 }  ,draw opacity=1 ] [dash pattern={on 3.75pt off 3pt on 7.5pt off 1.5pt}]  (214.87,287.88) -- (214.32,218.59) ;
\draw [line width=1.5]    (421.52,169.65) -- (499,135.01) ;
\draw [color={rgb, 255:red, 181; green, 0; blue, 0 }  ,draw opacity=1 ][line width=1.5]    (421.52,169.65) -- (493.43,137.55) ;
\draw [shift={(497.08,135.92)}, rotate = 155.94] [fill={rgb, 255:red, 181; green, 0; blue, 0 }  ,fill opacity=1 ][line width=0.08]  [draw opacity=0] (13.4,-6.43) -- (0,0) -- (13.4,6.44) -- (8.9,0) -- cycle    ;
\draw [color={rgb, 255:red, 207; green, 129; blue, 10 }  ,draw opacity=1 ][line width=1.5]    (242.51,129.92) -- (242.26,95.37) ;
\draw [shift={(242.23,91.37)}, rotate = 89.59] [fill={rgb, 255:red, 207; green, 129; blue, 10 }  ,fill opacity=1 ][line width=0.08]  [draw opacity=0] (13.4,-6.43) -- (0,0) -- (13.4,6.44) -- (8.9,0) -- cycle    ;
\draw  [dash pattern={on 3.75pt off 3pt on 7.5pt off 1.5pt}]  (191.68,128.07) -- (154.2,69.8) ;
\draw  [dash pattern={on 3.75pt off 3pt on 7.5pt off 1.5pt}]  (191.68,128.07) -- (191.4,70.2) ;
\draw [line width=1.5]  [dash pattern={on 2.25pt off 0.75pt on 0.75pt off 0.75pt}]  (231.94,287.22) .. controls (242.28,252.91) and (205.35,277.36) .. (194.65,295.78) ;
\draw [shift={(192.94,299.22)}, rotate = 292.07] [fill={rgb, 255:red, 0; green, 0; blue, 0 }  ][line width=0.08]  [draw opacity=0] (13.4,-6.43) -- (0,0) -- (13.4,6.44) -- (8.9,0) -- cycle    ;
\draw  [fill={rgb, 255:red, 5; green, 68; blue, 108 }  ,fill opacity=1 ] (212.68,294.46) .. controls (212.23,293.87) and (212.34,293.02) .. (212.94,292.57) -- (222.14,285.57) .. controls (222.73,285.12) and (223.57,285.23) .. (224.02,285.82) -- (226.47,289.04) .. controls (226.92,289.63) and (226.81,290.47) .. (226.22,290.92) -- (217.02,297.93) .. controls (216.42,298.38) and (215.58,298.27) .. (215.13,297.67) -- cycle ;
\draw  [fill={rgb, 255:red, 227; green, 224; blue, 224 }  ,fill opacity=1 ] (230.4,272.3) .. controls (248.53,258.9) and (249.19,276.93) .. (230.16,284.29) .. controls (211.12,291.66) and (219.61,284.77) .. (201.55,296.6) .. controls (183.49,308.43) and (186.93,316.48) .. (205,304.65) .. controls (223.06,292.82) and (212.27,285.71) .. (230.4,272.3) -- cycle ;

\draw  [dash pattern={on 4.5pt off 4.5pt}]  (408.4,169.76) -- (501.13,270.52) ;
\draw [line width=1.5]    (500.22,123.3) .. controls (471.78,105.91) and (501.89,111.5) .. (520.55,121.57) ;
\draw [shift={(523.97,123.55)}, rotate = 212.01] [fill={rgb, 255:red, 0; green, 0; blue, 0 }  ][line width=0.08]  [draw opacity=0] (13.4,-6.43) -- (0,0) -- (13.4,6.44) -- (8.9,0) -- cycle    ;
\draw  [fill={rgb, 255:red, 67; green, 173; blue, 40 }  ,fill opacity=1 ] (395.33,169.76) .. controls (395.33,162.99) and (401.19,157.51) .. (408.4,157.51) .. controls (415.62,157.51) and (421.47,162.99) .. (421.47,169.76) .. controls (421.47,176.52) and (415.62,182) .. (408.4,182) .. controls (401.19,182) and (395.33,176.52) .. (395.33,169.76) -- cycle ;
\draw [line width=1.5]    (325.46,167.66) -- (395.33,169.76) ;
\draw  [fill={rgb, 255:red, 5; green, 68; blue, 108 }  ,fill opacity=1 ] (317.85,169.46) .. controls (317.4,168.87) and (317.51,168.02) .. (318.1,167.57) -- (327.3,160.57) .. controls (327.89,160.12) and (328.74,160.23) .. (329.19,160.82) -- (331.64,164.04) .. controls (332.09,164.63) and (331.97,165.47) .. (331.38,165.92) -- (322.18,172.93) .. controls (321.59,173.38) and (320.75,173.27) .. (320.29,172.67) -- cycle ;
\draw  [fill={rgb, 255:red, 227; green, 224; blue, 224 }  ,fill opacity=1 ] (335.56,147.3) .. controls (353.69,133.9) and (354.35,151.93) .. (335.32,159.29) .. controls (316.29,166.66) and (324.78,159.77) .. (306.72,171.6) .. controls (288.66,183.43) and (292.1,191.48) .. (310.16,179.65) .. controls (328.22,167.82) and (317.43,160.71) .. (335.56,147.3) -- cycle ;

\draw [line width=1.5]  [dash pattern={on 2.25pt off 0.75pt on 0.75pt off 0.75pt}]  (324.66,155.9) .. controls (337.26,135.62) and (311.77,149.57) .. (300.49,163.31) ;
\draw [shift={(298.21,166.39)}, rotate = 303.19] [fill={rgb, 255:red, 0; green, 0; blue, 0 }  ][line width=0.08]  [draw opacity=0] (13.4,-6.43) -- (0,0) -- (13.4,6.44) -- (8.9,0) -- cycle    ;
\draw [color={rgb, 255:red, 47; green, 126; blue, 167 }  ,draw opacity=1 ][line width=1.5]    (322.82,163.15) -- (300.54,129.77) ;
\draw [shift={(298.32,126.45)}, rotate = 56.26] [fill={rgb, 255:red, 47; green, 126; blue, 167 }  ,fill opacity=1 ][line width=0.08]  [draw opacity=0] (11.61,-5.58) -- (0,0) -- (11.61,5.58) -- cycle    ;

\draw [line width=1.5]    (380.45,224.57) -- (406.98,181.23) ;
\draw  [fill={rgb, 255:red, 5; green, 68; blue, 108 }  ,fill opacity=1 ] (374.86,229.44) .. controls (374.24,229.04) and (374.07,228.2) .. (374.47,227.58) -- (380.78,217.89) .. controls (381.19,217.27) and (382.02,217.09) .. (382.65,217.5) -- (386.03,219.7) .. controls (386.66,220.11) and (386.83,220.94) .. (386.43,221.57) -- (380.11,231.26) .. controls (379.71,231.88) and (378.87,232.06) .. (378.25,231.65) -- cycle ;
\draw  [fill={rgb, 255:red, 227; green, 224; blue, 224 }  ,fill opacity=1 ] (384.11,202.63) .. controls (396.69,183.91) and (403.36,200.67) .. (387.91,214) .. controls (372.46,227.33) and (378.14,217.99) .. (365.1,235.2) .. controls (352.06,252.41) and (358.01,258.83) .. (371.05,241.62) .. controls (384.09,224.42) and (371.53,221.34) .. (384.11,202.63) -- cycle ;

\draw [line width=1.5]    (371.24,223.84) .. controls (348.24,247.96) and (360.06,219.72) .. (373.85,203.6) ;
\draw [shift={(376.5,200.68)}, rotate = 134.2] [fill={rgb, 255:red, 0; green, 0; blue, 0 }  ][line width=0.08]  [draw opacity=0] (13.4,-6.43) -- (0,0) -- (13.4,6.44) -- (8.9,0) -- cycle    ;
\draw [color={rgb, 255:red, 47; green, 126; blue, 167 }  ,draw opacity=1 ][line width=1.5]    (377.9,222.43) -- (338.03,206.47) ;
\draw [shift={(334.32,204.98)}, rotate = 21.82] [fill={rgb, 255:red, 47; green, 126; blue, 167 }  ,fill opacity=1 ][line width=0.08]  [draw opacity=0] (11.61,-5.58) -- (0,0) -- (11.61,5.58) -- cycle    ;

\draw  [dash pattern={on 0.84pt off 2.51pt}]  (359.51,206.15) .. controls (349.91,205.36) and (331.91,185.11) .. (332.31,167.75) ;
\draw  [fill={rgb, 255:red, 5; green, 68; blue, 108 }  ,fill opacity=1 ] (495.28,128.18) .. controls (495.6,127.51) and (496.4,127.22) .. (497.07,127.54) -- (507.52,132.48) .. controls (508.2,132.8) and (508.48,133.6) .. (508.17,134.28) -- (506.44,137.93) .. controls (506.12,138.6) and (505.32,138.89) .. (504.64,138.57) -- (494.19,133.62) .. controls (493.52,133.31) and (493.23,132.5) .. (493.55,131.83) -- cycle ;
\draw  [fill={rgb, 255:red, 227; green, 224; blue, 224 }  ,fill opacity=1 ] (523.1,133.72) .. controls (543.34,143.65) and (527.64,152.53) .. (512.34,139.02) .. controls (497.05,125.51) and (507.07,129.88) .. (488.26,119.29) .. controls (469.44,108.69) and (463.88,115.45) .. (482.69,126.04) .. controls (501.51,136.64) and (502.86,123.79) .. (523.1,133.72) -- cycle ;

\draw  [dash pattern={on 0.84pt off 2.51pt}]  (339.06,143.35) .. controls (360.26,100.55) and (435.46,106.95) .. (471.06,122.95) ;
\draw  [dash pattern={on 0.84pt off 2.51pt}]  (376.31,214.35) .. controls (420.91,231.11) and (518.66,192.86) .. (502.41,147.36) ;
\draw [color={rgb, 255:red, 47; green, 126; blue, 167 }  ,draw opacity=1 ][line width=1.5]    (503,129.51) -- (510.19,87.17) ;
\draw [shift={(510.86,83.23)}, rotate = 99.63] [fill={rgb, 255:red, 47; green, 126; blue, 167 }  ,fill opacity=1 ][line width=0.08]  [draw opacity=0] (11.61,-5.58) -- (0,0) -- (11.61,5.58) -- cycle    ;
\draw  [fill={rgb, 255:red, 255; green, 255; blue, 255 }  ,fill opacity=1 ] (400.66,169.49) .. controls (400.66,165.94) and (403.99,163.05) .. (408.11,163.05) .. controls (412.22,163.05) and (415.56,165.94) .. (415.56,169.49) .. controls (415.56,173.05) and (412.22,175.93) .. (408.11,175.93) .. controls (403.99,175.93) and (400.66,173.05) .. (400.66,169.49) -- cycle ; \draw   (400.66,169.49) -- (415.56,169.49) ; \draw   (408.11,163.05) -- (408.11,175.93) ;
\draw  [draw opacity=0][fill={rgb, 255:red, 0; green, 0; blue, 0 }  ,fill opacity=1 ] (407.97,163.19) .. controls (407.99,163.19) and (408,163.19) .. (408.01,163.19) .. controls (412.09,163.19) and (415.4,165.99) .. (415.42,169.46) -- (408.01,169.49) -- cycle ; \draw   (407.97,163.19) .. controls (407.99,163.19) and (408,163.19) .. (408.01,163.19) .. controls (412.09,163.19) and (415.4,165.99) .. (415.42,169.46) ;  
\draw  [draw opacity=0][fill={rgb, 255:red, 0; green, 0; blue, 0 }  ,fill opacity=1 ] (408.04,175.83) .. controls (408.03,175.83) and (408.02,175.83) .. (408,175.83) .. controls (403.93,175.79) and (400.65,172.96) .. (400.66,169.49) -- (408.07,169.53) -- cycle ; \draw   (408.04,175.83) .. controls (408.03,175.83) and (408.02,175.83) .. (408,175.83) .. controls (403.93,175.79) and (400.65,172.96) .. (400.66,169.49) ;  

\draw [color={rgb, 255:red, 207; green, 129; blue, 10 }  ,draw opacity=1 ][line width=1.5]    (408.5,169.29) -- (394.59,135.86) ;
\draw [shift={(393.05,132.16)}, rotate = 67.4] [fill={rgb, 255:red, 207; green, 129; blue, 10 }  ,fill opacity=1 ][line width=0.08]  [draw opacity=0] (13.4,-6.43) -- (0,0) -- (13.4,6.44) -- (8.9,0) -- cycle    ;
\draw [color={rgb, 255:red, 13; green, 154; blue, 144}  ,draw opacity=1 ][line width=1.5]    (421.52,169.65) -- (456.66,153.86) ;
\draw [shift={(460.31,152.23)}, rotate = 155.81] [fill={rgb, 255:red, 13; green, 154; blue, 144}  ,fill opacity=1 ][line width=0.08]  [draw opacity=0] (13.4,-6.43) -- (0,0) -- (13.4,6.44) -- (8.9,0) -- cycle    ;
\draw [color={rgb, 255:red, 9; green, 82; blue, 163 }  ,draw opacity=1 ][line width=1.5]    (416.4,180.26) -- (428.69,203.73) ;
\draw [shift={(430.55,207.27)}, rotate = 242.36] [fill={rgb, 255:red, 9; green, 82; blue, 163 }  ,fill opacity=1 ][line width=0.08]  [draw opacity=0] (13.4,-6.43) -- (0,0) -- (13.4,6.44) -- (8.9,0) -- cycle    ;
\draw [color={rgb, 255:red, 0; green, 0; blue, 0 }  ,draw opacity=1 ][line width=1.5]    (501.13,270.52) -- (500.94,243.35) ;
\draw [shift={(500.92,239.35)}, rotate = 89.6] [fill={rgb, 255:red, 0; green, 0; blue, 0 }  ,fill opacity=1 ][line width=0.08]  [draw opacity=0] (13.4,-6.43) -- (0,0) -- (13.4,6.44) -- (8.9,0) -- cycle    ;
\draw [color={rgb, 255:red, 0; green, 0; blue, 0 }  ,draw opacity=1 ][line width=1.5]    (500.13,271.02) -- (520.92,271.3) ;
\draw [shift={(524.92,271.35)}, rotate = 180.76] [fill={rgb, 255:red, 0; green, 0; blue, 0 }  ,fill opacity=1 ][line width=0.08]  [draw opacity=0] (13.4,-6.43) -- (0,0) -- (13.4,6.44) -- (8.9,0) -- cycle    ;
\draw [color={rgb, 255:red, 0; green, 0; blue, 0 }  ,draw opacity=1 ][line width=1.5]    (501.13,270.52) -- (490.02,288.45) ;
\draw [shift={(487.92,291.85)}, rotate = 301.78] [fill={rgb, 255:red, 0; green, 0; blue, 0 }  ,fill opacity=1 ][line width=0.08]  [draw opacity=0] (13.4,-6.43) -- (0,0) -- (13.4,6.44) -- (8.9,0) -- cycle    ;
\draw  [dash pattern={on 3.75pt off 3pt on 7.5pt off 1.5pt}]  (215.16,286.12) -- (181.33,237.67) ;
\draw [color={rgb, 255:red, 74; green, 74; blue, 74 }  ,draw opacity=1 ][line width=1.5]    (214.63,284.69) -- (191.73,252.28) ;
\draw [shift={(189.42,249.02)}, rotate = 54.74] [fill={rgb, 255:red, 74; green, 74; blue, 74 }  ,fill opacity=1 ][line width=0.08]  [draw opacity=0] (13.4,-6.43) -- (0,0) -- (13.4,6.44) -- (8.9,0) -- cycle    ;
\draw [color={rgb, 255:red, 207; green, 129; blue, 10 }  ,draw opacity=1 ][line width=1.5]    (214.63,284.69) -- (214.4,244.03) ;
\draw [shift={(214.38,240.03)}, rotate = 89.68] [fill={rgb, 255:red, 207; green, 129; blue, 10 }  ,fill opacity=1 ][line width=0.08]  [draw opacity=0] (13.4,-6.43) -- (0,0) -- (13.4,6.44) -- (8.9,0) -- cycle    ;
\draw [color={rgb, 255:red, 47; green, 126; blue, 167 }  ,draw opacity=1 ][line width=1.5]    (216.1,288.05) -- (201.98,266.65) ;
\draw [shift={(199.78,263.32)}, rotate = 56.59] [fill={rgb, 255:red, 47; green, 126; blue, 167 }  ,fill opacity=1 ][line width=0.08]  [draw opacity=0] (11.61,-5.58) -- (0,0) -- (11.61,5.58) -- cycle    ;
\draw [color={rgb, 255:red, 245; green, 166; blue, 35 }  ,draw opacity=1 ]   (184.92,241.77) .. controls (185.93,232.79) and (190.23,228.88) .. (199.52,229.94) .. controls (206.91,230.78) and (208.19,232.12) .. (210.93,233.92) ;
\draw [shift={(213.47,235.4)}, rotate = 206.95] [fill={rgb, 255:red, 245; green, 166; blue, 35 }  ,fill opacity=1 ][line width=0.08]  [draw opacity=0] (10.72,-5.15) -- (0,0) -- (10.72,5.15) -- (7.12,0) -- cycle    ;
\draw  [color={rgb, 255:red, 186; green, 0; blue, 0 }  ,draw opacity=1 ] (208.31,287.41) .. controls (208.31,283.56) and (211.43,280.44) .. (215.28,280.44) .. controls (219.13,280.44) and (222.26,283.56) .. (222.26,287.41) .. controls (222.26,291.26) and (219.13,294.38) .. (215.28,294.38) .. controls (211.43,294.38) and (208.31,291.26) .. (208.31,287.41) -- cycle ;
\draw  [color={rgb, 255:red, 186; green, 0; blue, 0 }  ,draw opacity=1 ][fill={rgb, 255:red, 186; green, 0; blue, 0 }  ,fill opacity=1 ] (213.06,287.41) .. controls (213.06,286.18) and (214.05,285.19) .. (215.28,285.19) .. controls (216.51,285.19) and (217.51,286.18) .. (217.51,287.41) .. controls (217.51,288.64) and (216.51,289.64) .. (215.28,289.64) .. controls (214.05,289.64) and (213.06,288.64) .. (213.06,287.41) -- cycle ;
\draw  [fill={rgb, 255:red, 255; green, 255; blue, 255 }  ,fill opacity=1 ] (235.06,129.92) .. controls (235.06,126.36) and (238.4,123.48) .. (242.51,123.48) .. controls (246.63,123.48) and (249.96,126.36) .. (249.96,129.92) .. controls (249.96,133.48) and (246.63,136.36) .. (242.51,136.36) .. controls (238.4,136.36) and (235.06,133.48) .. (235.06,129.92) -- cycle ; \draw   (235.06,129.92) -- (249.96,129.92) ; \draw   (242.51,123.48) -- (242.51,136.36) ;
\draw  [draw opacity=0][fill={rgb, 255:red, 0; green, 0; blue, 0 }  ,fill opacity=1 ] (242.38,123.62) .. controls (242.39,123.62) and (242.4,123.62) .. (242.42,123.62) .. controls (246.49,123.62) and (249.8,126.42) .. (249.83,129.89) -- (242.42,129.92) -- cycle ; \draw   (242.38,123.62) .. controls (242.39,123.62) and (242.4,123.62) .. (242.42,123.62) .. controls (246.49,123.62) and (249.8,126.42) .. (249.83,129.89) ;  
\draw  [draw opacity=0][fill={rgb, 255:red, 0; green, 0; blue, 0 }  ,fill opacity=1 ] (242.45,136.26) .. controls (242.44,136.26) and (242.42,136.26) .. (242.41,136.26) .. controls (238.33,136.22) and (235.05,133.39) .. (235.06,129.92) -- (242.47,129.96) -- cycle ; \draw   (242.45,136.26) .. controls (242.44,136.26) and (242.42,136.26) .. (242.41,136.26) .. controls (238.33,136.22) and (235.05,133.39) .. (235.06,129.92) ;  

\draw [line width=1.5]    (195.37,130.43) -- (242.51,129.92) ;
\draw  [fill={rgb, 255:red, 5; green, 68; blue, 108 }  ,fill opacity=1 ] (184.61,131.17) .. controls (184.16,130.58) and (184.27,129.74) .. (184.86,129.29) -- (194.06,122.28) .. controls (194.66,121.83) and (195.5,121.94) .. (195.95,122.54) -- (198.4,125.75) .. controls (198.85,126.34) and (198.74,127.19) .. (198.14,127.64) -- (188.94,134.64) .. controls (188.35,135.09) and (187.51,134.98) .. (187.06,134.39) -- cycle ;
\draw  [fill={rgb, 255:red, 227; green, 224; blue, 224 }  ,fill opacity=1 ] (202.32,109.02) .. controls (220.45,95.61) and (221.12,113.64) .. (202.08,121.01) .. controls (183.05,128.37) and (191.54,121.49) .. (173.48,133.32) .. controls (155.42,145.15) and (158.86,153.19) .. (176.92,141.36) .. controls (194.98,129.53) and (184.19,122.42) .. (202.32,109.02) -- cycle ;

\draw [line width=1.5]  [dash pattern={on 2.25pt off 0.75pt on 0.75pt off 0.75pt}]  (191.42,117.62) .. controls (204.03,97.33) and (178.53,111.28) .. (167.25,125.02) ;
\draw [shift={(164.97,128.1)}, rotate = 303.19] [fill={rgb, 255:red, 0; green, 0; blue, 0 }  ][line width=0.08]  [draw opacity=0] (13.4,-6.43) -- (0,0) -- (13.4,6.44) -- (8.9,0) -- cycle    ;
\draw [color={rgb, 255:red, 47; green, 126; blue, 167 }  ,draw opacity=1 ][line width=1.5]    (189.59,124.86) -- (167.3,91.49) ;
\draw [shift={(165.08,88.16)}, rotate = 56.26] [fill={rgb, 255:red, 47; green, 126; blue, 167 }  ,fill opacity=1 ][line width=0.08]  [draw opacity=0] (11.61,-5.58) -- (0,0) -- (11.61,5.58) -- cycle    ;

\draw [color={rgb, 255:red, 245; green, 166; blue, 35 }  ,draw opacity=1 ]   (163.26,82.55) .. controls (162.47,70.32) and (163.97,66.07) .. (170.97,65.07) .. controls (177.17,64.19) and (182.19,66.63) .. (187.92,71.14) ;
\draw [shift={(190.2,73)}, rotate = 220.16] [fill={rgb, 255:red, 245; green, 166; blue, 35 }  ,fill opacity=1 ][line width=0.08]  [draw opacity=0] (10.72,-5.15) -- (0,0) -- (10.72,5.15) -- (7.12,0) -- cycle    ;

\draw (175.9,209.56) node [anchor=north west][inner sep=0.75pt]  [color={rgb, 255:red, 245; green, 166; blue, 35 }  ,opacity=1 ]  {$\ai{i}$};
\draw (173.67,173.47) node [anchor=north west][inner sep=0.75pt]   [align=left] {a) Front view};
\draw (181.27,326.17) node [anchor=north west][inner sep=0.75pt]   [align=left] {b) Side view};
\draw (357.27,326.17) node [anchor=north west][inner sep=0.75pt]   [align=left] {c) Platform schematic};
\draw (225.01,301.24) node [anchor=north west][inner sep=0.75pt]  [color={rgb, 255:red, 186; green, 0; blue, 0 }  ,opacity=1 ]  {$\pPB{i}$};
\draw (221.51,230.09) node [anchor=north west][inner sep=0.75pt]  [color={rgb, 255:red, 207; green, 129; blue, 10 }  ,opacity=1 ]  {$z_{B}$};
\draw (165.51,243.09) node [anchor=north west][inner sep=0.75pt]    {$z_{p_{i}}$};
\draw (176.21,265.17) node [anchor=north west][inner sep=0.75pt]  [color={rgb, 255:red, 46; green, 127; blue, 167 }  ,opacity=1 ]  {$v_{i}$};
\draw (155.38,43.78) node [anchor=north west][inner sep=0.75pt]  [color={rgb, 255:red, 245; green, 166; blue, 35 }  ,opacity=1 ]  {$\betai{i}$};
\draw (233.78,67.92) node [anchor=north west][inner sep=0.75pt]  [color={rgb, 255:red, 207; green, 129; blue, 10 }  ,opacity=1 ]  {$z_{B}$};
\draw (410.17,139.25) node [anchor=north west][inner sep=0.75pt]    {$O_{B}$};
\draw (462.31,155.63) node [anchor=north west][inner sep=0.75pt]  [color={rgb, 255:red, 13; green, 154; blue, 144}  ,opacity=1 ]  {$x_{B}$};
\draw (409.51,216.42) node [anchor=north west][inner sep=0.75pt]  [color={rgb, 255:red, 9; green, 82; blue, 183 }  ,opacity=1 ]  {$y_{B}$};
\draw (421.17,114.47) node [anchor=north west][inner sep=0.75pt]    {$\mathcal{B}$};
\draw (373.37,114.49) node [anchor=north west][inner sep=0.75pt]  [color={rgb, 255:red, 207; green, 129; blue, 10 }  ,opacity=1 ]  {$z_{B}$};
\draw (455.85,124.4) node [anchor=north west][inner sep=0.75pt]  [color={rgb, 255:red, 181; green, 0; blue, 0}  ,opacity=1 ]  {$\pPB{1}$};
\draw (518.38,69.83) node [anchor=north west][inner sep=0.75pt]  [color={rgb, 255:red, 46; green, 127; blue, 167 }  ,opacity=1 ]  {$v_{1}$};
\draw (313.88,113.67) node [anchor=north west][inner sep=0.75pt]  [color={rgb, 255:red, 46; green, 127; blue, 167 }  ,opacity=1 ]  {$v_{2}$};
\draw (329.88,213.83) node [anchor=north west][inner sep=0.75pt]  [color={rgb, 255:red, 46; green, 127; blue, 167 }  ,opacity=1 ]  {$v_{n}$};
\draw (467.01,255.42) node [anchor=north west][inner sep=0.75pt]    {$O_{W}$};
\draw (491.51,221.42) node [anchor=north west][inner sep=0.75pt]    {$z_{W}$};
\draw (527.51,260.92) node [anchor=north west][inner sep=0.75pt]    {$x_{W}$};
\draw (469.51,293.25) node [anchor=north west][inner sep=0.75pt]    {$y_{W}$};
\draw (505.84,283.25) node [anchor=north west][inner sep=0.75pt]    {$\mathcal{W}$};

\end{tikzpicture}

\caption{This figure summarizes the main mathematical notations used in this paper.}
\label{fig:definitions}
\end{figure}
\FloatBarrier

\section{Modeling of Generic Tiltable MRAVs}\label{sec:modeling}

Let us first introduce a generic tiltable multi-rotor aerial design, depicted in figure~\ref{fig:definitions}\rev{c} as follows:

\begin{dfn}
A tiltable multirotor design is a rigid body, made from a fixed frame, actuated by a group of actuation units, each producing thrust, with the direction of at least one of its actuation units being controllable.
\end{dfn} 
To characterize such platforms, let us first denote by $\frameB$ as the body fixed frame, with origin $\originB$ placed at the platform's Center of Mass (CoM), with axes \{$\xB$, $\yB$, $\zB$\}. $\originB$ is considered to be floating in a fixed world frame denoted by $\frameW$, with origin $\originW$ and axes \{$\xW,$ $\yW$, $\zW$\}.\\
Then to characterize the actuation ability of each platform, let us first denote by $N \in \nN{}_{>0}$ as the number of actuation units. 
Each of these actuation units is attached to the platform at a fixed position with respect to (\wrt) $\frameB$ which we denote by ${\pPB{i} \in \nR{3}}$. 
In addition, each of these actuation units is expected to produce a thrust along its vertical axis; we denote by ${\vP{i} \in \nR{3}}$ as the corresponding thrust \wrt $\frameB$, with its norm $\uP{i} = \norm{\vP{i}} \in \nR{}_{>0}$, and its orientation $\vO{i} = \frac{\vP{i}}{\uP{i}}$, where we assumed all propellers to be uni-directional thrust propellers.\\
For convenience, let us define the orientation $\vO{i}$ with respect to the following angles depicted in figure~\ref{fig:definitions}:
\begin{enumerate}
\item radial rotation angle $\ai{i}$ \rev{(\fig\ref{fig:definitions}b)}: denotes the angle needed for a rotation about $\pPB{i}$ to bring $\zB$ into a vector $\zB'$ , where $\zB'$ is the projection of $\zB$
into the plane spanned by $\pPB{i}$ and $\vO{i}$;
\item tangential rotation angle $\betai{i}$\rev{(\fig\ref{fig:definitions}a)}: denotes the angle needed for a rotation about the axis perpendicular to the plane spanned by $\pPB{i}$ and $\vO{i}$ to let $\zB'$
coincide with $\vO{i}$.
\end{enumerate}
It is easy to see that each of these angles can be explicitly controlled by a servomotor.
Correspondingly, the thrust orientation $\vO{i}$ can be computed following the above rotation angles as follows:
\begin{align}\label{eq:h_space_to_v_space}
\vO{i} = \rotMat_{z}(\gamma_i)\rotMat_{y}(\betai{i}) \rotMat_{x}(\ai{i}) \vect{e}_3 = \rotMat_{z}(\gamma_i)\begin{bmatrix}
-s\betai{i} \\ s\ai{i} \; c\betai{i} \\ c\ai{i} \; c\betai{i}
\end{bmatrix}
\end{align} 
where $\rotMat_{z}$, $\rotMat_{y}$ and $\rotMat_{x}$ are the classical rotation matrices about the $z$, $y$ and $x$ axes, $\vect{e}_3$ is a unit vector such that $\vect{e}_3= [0 \; 0 \; 1]^T$,  $\gamma_i$ is the angle between $\xB$ and the projection of $\pPB{i}$ onto the plane \{$\xB, \yB$\}, and where for any angle $\bullet$, $s\bullet = sin(\bullet)$ and $c\bullet = cos(\bullet)$.

For ease of notation, let us denote by $\VP = [\vP{1}^T \dots \vP{N}^T]^T \in \vSet \subset \nR{3N}$ as the concatenation matrix of all forces produced by the corresponding propellers in $\nR{3}$, and let us denote by $\thrusts = [\uP{1} \dots \uP{N}]^T \in \thrustSet \subset \nR{N}$ and ${\thrustsDer = [\uPDer{1} \dots \uPDer{N}]^T \in \thrustSetDer \subset \nR{N}}$ as the concatenation matrices of the propellers' thrusts and the corresponding derivatives. Similarly, let us denote by 
${\orientP = [\ai{1} \; \betai{1} \dots \ai{N}\; \betai{N}]^T \in \orientSet \subset \nR{2 N}}$ as the concatenation of the orientation angles of all actuation units, and by\\ 
${\orientPDer = [\aid{1} \; \bid{1} \dots \aid{N}\; \bid{N}]^T \in \orientSetDer \subset \nR{2 N}}$ as the corresponding derivatives.
Finally, let us denote by a control input \\${\vect{H} = [\thrusts^T \orientP^T]^T \in \cntrlSet} \rev{\subset \nR{N+2N}}$, where $\cntrlSet$ is referred to as the control set, and let us denote by $\cntrlDer \rev{\subset \nR{N+2N}}$ as the corresponding derivative. For convenience, let us denote by Degrees of Freedom (DoF) as the dimension of $\cntrlSet$, \ie the number of active control inputs.\\
While the above notation presents two sets for the platform's control inputs, it should be noted that there is \rev{a map between $\VP$ and $\vect{H}$ following \eqn\eqref{eq:h_space_to_v_space}. Note that the mapping $\vect{H} \rightarrow \VP$ is injective; however, it is not bijective due to the singularities in \eqn\eqref{eq:h_space_to_v_space}}. Note that the different notations will be used interchangeable throughout the paper.

Each of the above described actuation units is expected to produce a force and a moment about its vertical axis $\vO{i}$. The resultant wrench $\wrench \in \nR{6}$ applied at the CoM of the platform can be expressed as follows:
\begin{align}
\wrench(\thrusts,\orientP) = 
\begin{bmatrix}
\fB(\thrusts,\orientP)\\
\mB(\thrusts,\orientP)
\end{bmatrix} = 
\sum_{i=1}^N 
\begin{bmatrix}
\vP{i} \\
(\skew{\pPB{i}} + r_i ) \vP{i}
\end{bmatrix}
\end{align} 
where $\fB \in \nR{3}$ and $\mB \in \nR{3}$ are the total forces and moments applied on the CoM of the platform, $\skew{\cdot}$ denotes the skew-symmetric map $\nR{3} \rightarrow SO(3)$, and $r_i = \pm \frac{c_\tau}{c_T}$ is the ratio between the propeller's drag and thrust coefficients, respectively $c_\tau$ and $c_T$, with its sign being relative to the corresponding propeller rotation direction. \rev{In this paper, both lift and drag coefficients are assumed to be constant similar to \cite{pounds2002design}, where each propeller generates lift and drag along its axis of rotation $\vO{i}$ proportional to its rotational speed as follows:
\begin{align}
    \uP{i} = c_T \omega_i^2 \\
    \tau_i = C_\tau \omega_i^2
\end{align}
}

\subsection{Allocation Strategies}\label{subsec:allocation_strategy}
In this paper, we refer to an allocation strategy, or allocation map, as any mapping between the platform's control inputs and the applied wrench. Since we are introducing two control inputs, such mapping could be in one of the following forms, \rev{which will be discussed below}:
\begin{enumerate}
    \item $\vect{H} \rightarrow \totalWrench$; or
    \item $\VP \rightarrow \totalWrench.$
\end{enumerate}
\rev{First, to illustrate the effect of the allocated wrench $\totalWrench$ on the platform's dynamics, let us write the dynamic equation of motion of a multirotor platform following the Newton-Euler formalism as follows}:
\begin{align}\label{eq:newton_euler}
\begin{bmatrix}
\massB \ddpB \\
\inertiaB \angAccB
\end{bmatrix}
&= -\begin{bmatrix}
\massB g \zW\\
\angVelB \times \inertiaB \angVelB
\end{bmatrix}
+ \vect{G} \wrench(\thrusts,\orientP)
\end{align}
where $\massB \in \nR{}_{>0}$ and $\inertiaB \in \nR{3 \times 3}_{>0}$ are the platform mass and inertia tensor with respect to $\frameB$, $g$ is the gravitational constant, \rev{$\ddpB ,\angAccB , \angVelB \in \nR{3}$ are respectively the platform linear acceleration in $\frameW$ and the angular acceleration and velocity in $\frameB$}, and $\vect{G}$ is a transformation matrix as follows:
\begin{align}
\vect{G} = 
\begin{bmatrix}
\eye{3}  & \vect{0}_{3 \times 3} \\
\vect{0}_{3\times 3}  & \rotMatB
\end{bmatrix}
\end{align}
where $\rotMatB$ is the rotation matrix that brings a vector in $\frameB$ to $\frameW$.\\
\rev{Let us first look at the mapping from $\vect{H}$ to the platform wrench $\totalWrench$. The authors in \cite{hamandi2021ijrr} show that this map for a platform with fixed propellers is as follows:}
\begin{align}\label{eq:allocation_derivative}
\wrench (\thrusts, \orientP) = 
\vect{F}_{fixed}(\orientP) \thrusts = 
\at{\frac{\delta \wrench}{\delta \thrusts}}{\orientP} 
\thrusts
\end{align}
\rev{where $\vect{F}_{fixed}(\orientP)$ is the fixed full allocation matrix; this allocation matrix is fixed for any control inputs, however, it is only valid for platforms with fixed propellers. The authors in \cite{hamandi2021ijrr} also show that the allocation matrix in \eqn\eqref{eq:allocation_derivative} is a special case of the corresponding full allocation matrix $\vect{F}(\vect{H})$ for a platform with tiltable propellers. However, the generic allocation matrix maps the derivative of the control input to the derivative of the wrench as follows:}
\begin{align}\label{eq:allocation_derivative_d}
\wrenchD (\thrusts, \orientP, \thrustsD, \orientPD) = 
\vect{F}(\vect{H}) \begin{bmatrix}
\thrustsD\\
\orientPD
\end{bmatrix} = 
\left[
\frac{\delta \wrench}{\delta \thrusts} \; \frac{\delta \wrench}{\delta \orientP}
\right]_{\begin{pmatrix} {\scriptstyle \thrusts} \\ {\scriptstyle \orientP} \end{pmatrix}} 
\begin{bmatrix}
\thrustsD\\
\orientPD
\end{bmatrix}
\end{align}

\rev{Now let us look at the map between $\VP$ and $\totalWrench$. Examples of such a map have been introduced for different platforms in the literature. Let us first introduce this map for the case where each propeller can be tilted in $\mathbb{S}^2$ as presented in \cite{ramp2015modeling}:}

\begin{align}\label{eq:allocation_all}
\wrench (\thrusts, \orientP) = 
\vect{A} \VP = 
\begin{bmatrix}
\eye{3} &\dots &\eye{3}\\
\underbrace{\skew{\pPB{1}} + r_1}_{\vect{A}_1} &\dots & \underbrace{\skew{\pPB{N}} + r_N}_{\vect{A}_N}  
\end{bmatrix}
\VP
\end{align}

where $\vect{A}$ is referred to as the vector allocation matrix.
\rev{N}ote that the allocation strategy in \eqn\eqref{eq:allocation_all} can only be used when each propeller's direction can be \rev{modified in $\mathbb{S}^2$}. \rev{Other formulations presented in the literature can reflect partial control of the propellers' orientation, such as the one presented for the hexarotor with tilting propellers in~\cite{bodie2020iser}}. \rev{In what follows, we present a general formulation for any MRAV where each propeller's orientation can be fixed, tiltable in $\mathbb{S}^1$ or $\mathbb{S}^2$}.\\
Let us modify the matrix $\vect{A}$ depending on whether each of its propellers can be tilted about two axes, one axis or no axes. In the latter, the corresponding vector $\vO{i}$ is fixed, and thus each matrix $\vect{A}_i$ can be replaced with the single column matrix $\vect{A}_i' = \vect{A}_i \vO{i}$, while the control input $\vP{i}$ can be replaced by the propeller thrust $\uP{i}$.\\
Similarly, if only one rotation axis of the corresponding propeller can be actuated, the corresponding matrix $\vect{A}_i$ can be replaced with a matrix $\vect{A}_i' = \vect{A}_i \vect{W}$, such that $\vect{A}_i \vect{W} \in \nR{6 \times 2}$, and $\vP{i}$ can be replaced with a vector $\vP{i}' \in \nR{2}$.
For convenience, let us denote by $\vect{A}' = [\vect{A}_1' \; \dots \; \vect{A}_N']$ as the modified vector allocation matrix.
\rev{For example, a platform where the first propeller can be tilted about two axes, the second about one axis (such that $\beta = 0$), and the third is fixed (oriented parallel to $\zB$) can be written as follows:}
\rev{\begin{align}
    \wrench (\thrusts, \orientP) = 
\vect{A}' \VP' = 
\begin{bmatrix}
\eye{3} &\begin{bmatrix}
\vect{0}\\
\eye{2}
\end{bmatrix} &\begin{bmatrix}
0\\0\\1
\end{bmatrix}\\
\underbrace{\skew{\pPB{1}} + r_1}_{\vect{A}'_1} &
 \underbrace{(\skew{\pPB{N}} + r_N)\begin{bmatrix}
\vect{0}\\ \eye{2}
\end{bmatrix}}_{\vect{A}_2}& \underbrace{(\skew{\pPB{N}} + r_N)\begin{bmatrix}
0\\0\\1
\end{bmatrix}}_{\vect{A}_3}  
\end{bmatrix}
\begin{bmatrix}
\vP{1}\\
\uP{2}^1 \\ \; \uP{2}^2\\
\uP{3}
\end{bmatrix}
\end{align}}
\rev{where $[\uP{2}^1\; \uP{2}^2]^T \in \nR{2}$ can be computed from \eqn\eqref{eq:h_space_to_v_space} as follows $\uP{2}^1 =s\ai{2}\uP{1}$ and $\uP{2}^1 =c\ai{2}\uP{1}$.} 
\begin{rmk} In the case of fixed propellers, the matrices $\vect{A}'$ and $\allocationMatrix$ are identical. 
\end{rmk}

Finally, let us define the control input corresponding to $\vect{A}'$ as $\VP' \in \vSet'$, and define the set of feasible wrench as follows:

\begin{align}
    \wrenchSet = \{ \totalWrench = \vect{A}'  \VP' \;|\; \forall \VP' \in \vSet'\}
\end{align}
Similarly, we define the force set at hover $\forceSet$ as the set of forces that can be applied by the platform while applying zero-moment:
\begin{align}\label{eq:force_set}
    \forceSet = \{ \vect{f} = \vect{A}_f' \VP' \;|\; \forall \VP' \in \vSet' \;  s.t. \; \vect{A}_m' \VP' = \vect{0} \}
\end{align}
where $\vect{A}_f'$ and $\vect{A}_m'$ are correspondingly the first and last three rows of $\vect{A}'$, which correspond to the map between the control inputs and the applied forces and moments.
Finally, we define the moment set at hover $\momentSet$ as the set of moments that can be applied by the platform while applying a force that counteracts gravity:
\begin{align}
    \momentSet = \{ \vect{m} = \vect{A}_m' \VP' \;|\; \forall \VP' \in \vSet' \;  s.t. \; \norm{\vect{A}_f' \VP'} \geq \massB g \}
\end{align}

\rev{While the authors in \cite{hamandi2021ijrr} suggest to analyze the full allocation matrix $\allocationMatrix$ to understand the properties and abilities of any MRAV, it is easy to see that the full allocation matrix changes depending on the operating point, and as such, in the case of tiltable propellers ${\Image(\allocationMatrix) :\neq \wrenchSet}$, where $\wrenchSet$ is the set of feasible wrench. In fact, the set of feasible wrench can be described as follows: ${\wrenchSet} := \{\Image(\allocationMatrix(\vect{H}) \; \forall \vect{H} \in \cntrlSet\}$, where $\allocationMatrix(\vect{H})$ is the full allocation matrix from \eqn\eqref{eq:allocation_derivative_d} parametrized by the control input $\vect{H}$. On the other hand, it is easy to see that for any platform (fixed or actuated) ${\Image(\vect{A}') := \wrenchSet}$, and as such, in this paper we will rely on the matrix $\vect{A}'$ for the analysis of the platform's ability to achieve and sustain static hovering.}

\section{STATIC HOVERING}\label{sec:static_hovering}

In what follows we will focus on the hovering ability of a tiltable MRAV.
Let us first define the static hovering ability, as the ability of the platform to stabilize its position at a certain altitude, with negligible linear and angular velocities as follows:

\begin{align}\label{eq:hovering_main}
\pB \rightarrow \pBd,\; \rotMatB \rightarrow \rotMatB^h;\; \dpB \rightarrow \vect{0};\;   \angVelB \rightarrow \vect{0}
\end{align} 
where $\rotMatB^h$ is a constant hovering orientation such that $\rotMatB^h \in \mathcal{R}$, where $\mathcal{R}$ is the set of orientations for which static hovering is possible.\\
Static hovering should be differentiated from dynamic hovering, where the platform can reach a desired position while varying its orientation, and with a non negligible linear and angular velocity. An example of a platform that can achieve dynamic hovering was presented in \cite{zhang2016icra}.

Conditions for static hovering were presented in \cite{michieletto2017} for platforms with fixed propellers. Since $\Image(\vect{A}') := \wrenchSet$, conditions similar to \cite{michieletto2017} can be derived while relying on the vector allocation matrix $\vect{A}'$ as follows: 
\begin{align} \label{eq:hover_1}
\rank (\vect{A}_m') = 3\\\label{eq:hover_2}
\exists \VP' \in \interior(\vSet') \; s.t. \;\; \bigg\{\begin{matrix}
\norm{\vect{A}_f' \VP'} &\geq& \massB g \\
\vect{A}_m' \VP' &=& \vect{0}
\end{matrix}
\end{align}
where following \eqn\eqref{eq:hover_2}, if the platform is able to apply a force magnitude larger than its weight, the platform will be able to achieve static hovering at any orientation $\mathcal{R}$ such that:
\begin{align}\label{eq:hovering_orientations}
    \mathcal{R} = \{\rotMatB^h \; | \; \vect{A}_f' \rev{ \vect{B}_m'}\VP' \geq \rotMatB^h \rev{\vect{\zW}} \massB g  \}
\end{align}
\rev{where $ \vect{B}_m'$ is the basis of the null space of $ \vect{A}_m'$.}
\rev{Let us note that $\forall \rotMatB^T \in \mathcal{R}$, $\rotMatB^T \vect{A}_f' \vect{B}_m'\vect{V}' \propto \zW$, and define the set of hovering orientations as follows:
\begin{align}
    \mathcal{Z} = \{\vect{d} \in \mathbb{S}^2 \; \text{s.t.} \; \vect{d} = \rotMatB\zW \propto  \vect{A}_f' \vect{B}_m' \vect{V}' \;\; \forall \rotMatB \in \mathcal{R} \}
\end{align}
\begin{rmk}
Note that multiple platform orientations $\rotMatB^h \in \mathcal{R}$ could correspond to the same hovering orientation $\vect{d} \in \mathcal{Z}$. For example, for a classical quadrotor, $\vect{d} = \zB$ corresponds to the set of rotation matrices $\rotMat_z(\gamma_i)\rotMatB^h \forall \gamma_i \; \in [-\pi,\pi]$.
\end{rmk}}
The conditions \eqrefs{eq:hover_1}{eq:hover_2} are equivalent to the condition below \cite{baskaya2021ral}:
\begin{property}
A platform is capable of static hovering if and only if $\vect{0} \in \interior(\momentSet)$
\end{property}
While intuitively for a platform to be able to hover it should only be able to lift its weight while applying zero moment, the condition of $\vect{0} \in \interior(\momentSet)$ is necessary so that the platform can reorient itself and counteract any disturbances.

As seen above, the $\vect{A}'$ matrix allows the analysis of all the forces and moments the platform could reach, and correspondingly, allows the analysis of the hovering ability of the platform.
However, with platforms with tiltable propellers, it is important to take into consideration the propeller re-orientation dynamics. As such, it is important to understand, at each hovering orientation $\rotMatB^h$ the ability of the platform to reorient itself and counteract disturbances, which will be discussed in subsection \ref{subs:local_hover}.

\subsection{Static Hovering Realization}

In this subsection, we aim to characterize the possible static hovering orientations, \ie the dimensionality of \rev{$\mathcal{Z}$}. The authors in \cite{hamandi2021ijrr} claimed that for any platform, \rev{$\mathcal{Z}$} can be fully characterized from the full allocation matrix $\allocationMatrix$, and the set of allowable control inputs $\cntrlSet$. As discussed in section~\ref{subsec:allocation_strategy}, while $\Image(\allocationMatrix_{fixed}) := \wrenchSet$ in the case of fixed propellers, $\Image(\allocationMatrix) :\neq \wrenchSet$ in the case of tiltable propellers. Hence, in what follows we will develop the conditions for static hovering realization for any MRAV while relying \rev{on the} vector allocation matrix $\vect{A}'$.

First, let us group platforms' based on the dimensionality of \rev{$\mathcal{Z}$} as follows:
\begin{enumerate}
\item Unidirectional Thrust (UDT)
\item Multidirectional Thrust (MDT)
\item Fully Actuated (FA)
\item Omnidirectionl (OD)
\end{enumerate}


\rev{Let us note that each platform is assumed to be able to apply moments in $\nR{3}$ to balance its attitude. A necessary condition for this assumption is as follows:}
\rev{\begin{ass}
Any MRAV studied in this paper can change its attitude in $\nR{3}$, with its vector allocation matrix satisfying:
\begin{align}
\rank(\vect{A_m}') = 3
\end{align}
\end{ass}
}
\rev{
\begin{prop}\label{prop:directions_of_hover}
$\dim{(\mathcal{Z})} \leq \rank(\vect{A}_f')$.
\end{prop}
\begin{proof}
The proof of proposition \ref{prop:directions_of_hover} follows from the definition of $\mathcal{Z}$ as follows:\\ ${\dim{(\mathcal{Z})} = \dim(\Image (\vect{A}_f' \vect{B}_m')) \leq \dim(\Image (\vect{A}_f')) = \rank(\vect{A}_f')}$.
\end{proof}
Following proposition \ref{prop:directions_of_hover}, we can define the properties of the platform groups as follows:
}

\begin{dfn}\label{dfn:udt}
A platform is said to be a unidirectional thrust platform (UDT) if its thrust can be varied along one direction while hovering as follows:
\begin{align}
\rank(\vect{A}_f') = 1 \; \text{and} \; \rank(\vect{A}') = 4
\end{align}
Correspondingly, this platform has a unique hovering orientation \rev{direction}.
\end{dfn}

\begin{dfn}\label{dfn:mdt}
A platform is said to be a multidirectional thrust platform (MDT) if its thrust can be varied in a plane while hovering as follows:
\begin{align}
2 \leq \rank(\vect{A}_f') \leq 3 \; \text{and} \; 5 \leq \rank(\vect{A}') \leq 6
\end{align}
Correspondingly, this platform can achieve static hovering in a plane.
\end{dfn}

\begin{dfn}\label{dfn:FA}
A platform is said to be a fully actuated platform (FA) if its thrust can be varied in $\nR{3}$ while hovering as follows:
\begin{align}
\rank(\vect{A}_f') = 3 \; \text{and} \; \rank(\vect{A}') = 6
\end{align}
Correspondingly, this platform can achieve static hovering in a variety of orientations in $\nR{3}$.
\end{dfn}

\begin{dfn}\label{dfn:OD}
A platform is said to be an omnidirectional platform (OD) if it can hover in all directions in a sphere around its CoM as follows:
\begin{align}\nonumber
\rank(\vect{A}_f') = 3 \; \text{and} \; \rank(\vect{A}') = 6\\ \label{eq:OD_cond_1}
and \; \forall \bar{\vect{f}} \in \text{ball}_{\nR{3}}(\vect{0},\massB g), \; \bar{\vect{f}} \in \interior(\forceSet) 
\end{align}
\end{dfn}
where $\text{ball}_{\nR{3}}(\vect{0},\massB g)$ refers to a topological ball centered about the origin $\vect{0}$ with radius $\massB g$.

As proposed in \cite{hamandi2021airpharo}, to understand numerically if a platform is OD (\ie satisfies condition \eqref{eq:OD_cond_1}), it is enough to compare the platform's Omnidirectional Lift (ODL) to its weight.
The ODL was defined as the radius of the maximum inscribed sphere inside of the platform's set of feasible forces $\forceSet$. As such, condition (\ref{eq:OD_cond_1}) can be modified as follows:
\begin{dfn}
A platform is said to be an omnidirectional platform (OD) iff:
\begin{align}\nonumber
\rank(\vect{A}_f') = 3 \; &\text{and} \; \rank(\vect{A}') = 6\\
\text{ODL} &\geq \massB g
\end{align}
\end{dfn}

As seen above, to understand the dimensionality of \rev{$\mathcal{Z}$} for a MRAV, it is enough to analyze it's vector allocation matrix $\vect{A}'$. On the other hand, the OD property requires analysis of the feasible force set in addition to the analysis of the vector allocation matrix.
However, note that the analysis of the dimensionality of \rev{$\mathcal{Z}$} assumed that the corresponding platform can achieve static hovering; as shown in \eqns\eqrefs{eq:hover_1}{eq:hover_2}, one of the conditions for static hovering is the platform's ability to apply forces larger than its weight. Consequently, while the definitions \eqref{dfn:mdt}-\eqref{dfn:FA} do not explicitly analyze the platform's set of feasible controls $\cntrlSet$, the necessary hovering ability  inherently requires the analysis of the platform's actuation limits.

\subsection{Static Hovering Sustainability}\label{subs:local_hover}

In the case of platforms with tiltable propellers, we refer to static hovering sustainability as the platform's ability to maintain already achieved static hovering. In this section we will analyze the ability of the platform to hover in the neighborhood of a hovering orientation, the ability of the platform to maintain hovering with fixed propellers' orientations, and the ability of the platform to change from one hovering orientation to another.

First, let us assume that for any \rev{hovering orientation $\vect{d}_0 \in \mathcal{Z}$}, there exists a control input ${\vect{H}_0 = (\thrusts_0, \orientP_0) \in \interior(\cntrlSet)}$ that achieves the hovering conditions \eqrefs{eq:hover_1}{eq:hover_2}. Note that $\vect{H}_0 \in \interior(\cntrlSet)$ can be uniquely mapped to a $\VP_0' \in \interior(\vSet')$.
\rev{Then, let us define the set $\mathcal{Z}_0$ as the set of hovering orientations in the neighborhood of $\vect{d}_0$ as follows:
\begin{align}
    \mathcal{Z}_0 = \{\vect{d} \in \mathbb{S}^2 \;\text{s.t}\; \vect{d} \in \text{ball}_{\mathbb{S}^2}(\vect{d}_0,\eta)  \}
\end{align}
where $\eta$ represents a small orientation change.
}
\rev{\begin{prop}\label{prop:hoverability}
If the platform is FA, then $\exists \eta$ such that static hovering can be achieved $\forall \vect{d} \in \mathcal{Z}_0$ .
\end{prop}
\begin{proof}
First, let us note that for any control input $\vect{V}' \in \vSet'$, 
for a FA platform, $\Image (\vect{A}_f' \vect{B}_m') \in \nR{3}$. Then $\forall \vect{d} \in \mathbb{S}^2$,\\ ${\exists\vect{V}' \in \vSet'\; \text{s.t.} \;\vect{d} \propto \vect{A}_f' \vect{B}_m' \vect{V}'}$. 
Then, let us decompose $\vect{d} \in \mathcal{Z}_0$ as $\vect{d} = \vect{d}_0 + \delta \vect{d}$. Static hovering can be achieved for any $\eta$ such that:
\begin{align}\label{eq:hovering_change}
    \vect{A}_f' \vect{B}_m'(\vect{V}' + \delta \vect{V}') \geq \massB g (\vect{d} + \delta \vect{d})\\\nonumber
    \text{s.t.} \vect{V}' + \delta \vect{V}' \in \vSet' \;\; \text{and} \; \; \delta \vect{d} \in \text{ball}_{\mathbb{S}^2}(\vect{0},\eta)
\end{align}
where $\delta \vect{V}'$ is a change in the control input $\vect{V}'$ required to make the change in hovering orientation.\\
Let us assume $ {\vect{A}_f' \vect{B}_m'\vect{V}' =  \massB g \vect{d}}$, and rewrite \eqn\eqref{eq:hovering_change} as follows:
\begin{align}\label{eq:hovering_change_2}
    \vect{A}_f' \vect{B}_m' \delta \vect{V}' \geq \massB g \delta \vect{d}\\\nonumber
    \text{s.t.} \vect{V}' + \delta \vect{V}' \in \vSet' \;\; \text{and} \; \; \delta \vect{d} \in \text{ball}_{\mathbb{S}^2}(\vect{0},\eta)
\end{align}
Since the platform if FA, and from the continuity of $\vSet'$, there exists $\eta$ that satisfied \eqn\eqref{eq:hovering_change_2}.
\end{proof}
While proposition \ref{prop:hoverability} holds for any FA platform, such guarantees do not exist for UDT or MDT platforms.\\
}




\rev{While FA platforms can sustain static hovering for any orientation in $\mathcal{Z}_0$, it is important to analyze the ability of MRAV with tiltable propellers to sustain static hovering after the failure of their servomorotrs, \ie if the propellers' orientations are fixed at $\orientP_0$.}
\begin{prop}\label{prop:hover_n_3}
\rev{Any MRAV} platform with $N \leq 3$ cannot sustain static hovering with fixed propellers' orientations $\orientP =\orientP_0$.
\end{prop}

\begin{proof}
With fixed propellers, the matrix $\vect{A}'$ is reduced such that ${\vect{A}_{fixed}' = \vect{F}_{fixed} \in \nR{6 \times N}}$ with a corresponding controls input $\thrusts$; as such $\rank(\vect{A}_{fixed}') \leq N$. 
With a control input $\vect{H}_0$, the platform is capable of applying a force that can counteract gravity, while applying zero-moment. 
It is easy to see that any control input $\thrusts = k \thrusts_0$, with $k \in \nR{}$ being a random scalar, can still achieve the hovering conditions \eqref{eq:hover_2} \rev{necessary to achieve hovering}.
Accordingly, at a fixed $\orientP_0$, $\rank(\vect{A}_{f}' \vect{B}_{m}) \geq 1$.\\
Rewriting the above inequality, 
\begin{align}
rank(\vect{A}_{f}' \vect{B}_{m})  = N -rank(\vect{A}_m)' 
- dim\left(ker(\vect{A}_f')\cap ker(\vect{A}_m')\right) \geq 1 \\
rank(\vect{A}_m') \leq N - 1 - dim\left(ker(\vect{A}_f')\cap ker(\vect{A}_m')\right)
\end{align}

In the case where $N \leq 3$, $\rank(\vect{A}_m') \leq 2$, and as such, the platform cannot sustain static hovering with fixed propellers' orientations $\orientP =\orientP_0$ \rev{as it cannot achieve the conditions from \eqns\eqrefs{eq:hover_1}{eq:hover_2} at the same time.}
\end{proof}

\rev{Following proposition \ref{prop:hoverability} and the consequent proof, it is important to analyze for any MRAV with tiltable propellers the effect of changes in its control input $\vect{V}'$ on the applied wrench. As such, let us define the following sets:}
\begin{align}
\vSet_0' = \{ \forall \VP' | \VP' \in \vSet' \; and \; \VP' \in \text{ball}(\VP_0', \epsilon)\}
\end{align} 
\begin{align}\label{eq:w_0_def_A}
\wrenchSet_0 = \{ \forall \totalWrench \in \wrenchSet \; s.t. \; \totalWrench = \vect{A}'\VP' \;  \forall \VP' \in\vSet_0' \}
\end{align}  
where $\text{ball}_{\nR{3}}(\VP_0', \epsilon)$ refers to a topological ball centered about the origin $\VP_0'$ with radius $\epsilon$ and where $\epsilon$ is a small tolerance. \rev{For convenience, let us denote by $\momentSet_0$ and $\forceSet_0$ as the set of moments and forces corresponding to $\wrenchSet_0$.}\\

To characterize the stability of the platform's static hovering at a hovering orientation $\rotMatB^h \in \mathcal{R}$, it is required to understand the force and moment components of the set $\wrenchSet_0$; for simplicity, let us assume that the produced force inside $\wrenchSet_0$ to be relatively close to the hovering force as assumed in \rev{\cite{baskaya2021ral}}. However, the analysis of the feasible moments inside the corresponding set is still required to understand if the platform can change its hovering orientation and to counteract external disturbances.
Let us denote the moment set corresponding to  $\wrenchSet_0$ by $\momentSetLocal$.

To be able to numerically characterize the stability of each hovering orientation $\rotMatB^h \in \mathcal{R}$, let us introduce the Local Hoverability Index (LHI) as the maximum moment the platform can apply in any direction at any hovering orientation $\rotMatB^h \in \mathcal{R}$: 
\begin{align}
\text{LHI} = \argmax_d \; \text{ball}_{\nR{3}}(\vect{0},d) \in \momentSetLocal
\end{align}

To ease the computation of the LHI, in what follows we will discuss how to compute it from the full allocation matrix $\allocationMatrix$ and the derivative of the control inputs $\cntrlDer$.

\begin{prop}\label{prop:rA_rF}
$\forall \VP' \in \interior(\vSet_0')$\rev{, as $\epsilon \rightarrow 0$}, $\rank(\allocationMatrix) = \rank(\vect{A}')$.
\end{prop}

\begin{proof}
\rev{As $\epsilon \rightarrow 0$, we can linearize the applied wrench with respect to $\vect{H}_0$ as follows:
\begin{align}
    \totalWrench &= \totalWrench_0 + \left. \frac{d \totalWrench}{d \vect{H}}\right|_{\vect{H}_0}(\vect{H} - \vect{H}_0) \\ \nonumber
    \totalWrench &= \totalWrench_0 + \allocationMatrix(\vect{H}_0)\; \dot{\vect{H}} dt 
\end{align}
where $\totalWrench_0$ is the wrench applied by the control input $\vect{V}_0'$.
Correspondingly, }$\wrenchSet_0$ can be expressed as a codomain of $\cntrlDer$ as follows:
\begin{align}\label{eq:w_0_def_F}
\wrenchSet_0 = \{\totalWrench | \totalWrench = \rev{\totalWrench_0} + \allocationMatrix(\vect{H}_0) \dot{\vect{H}} dt \; \forall \dot{\vect{H}} \in \cntrlDer \}
\end{align}
where $dt \in \nR{}_{>0}$ is a time constant used for calculation purposes.
Consequently, we can see from \eqn\eqref{eq:w_0_def_A} and \eqn\eqref{eq:w_0_def_F} that 
$\wrenchSet_0 := \Image(\allocationMatrix(\vect{H}_0)) = \Image(\vect{A}') \; \forall \rotMatB^h \in \mathcal{R}$, and as such, proposition \ref{prop:rA_rF} holds.
\end{proof}
Following the above proposition, and the corresponding proof, we can see that computing the LHI from $\allocationMatrix$ provides the same $\momentSetLocal$ that could have been computed from $\vect{A}'$.

Finally, we assume $\cntrlDer$ to be fixed for each actuator type. While this assumption and the introduction of \eqn\eqref{eq:w_0_def_F} allow for ease of computation of $\wrenchSet_0$, we do note that in reality $\cntrlDer$ changes depending on $\vect{H}_0$. Consequently, the corresponding LHI can only be used as a relative metric that compares the hovering stability of one platform with respect to another, or the hovering stability of different hovering orientations $\rotMatB^h \in \mathcal{R}$ for the same platform. 

\section{Critically Statically Hoverable Platforms}\label{sec:csh}

In what follows, we refer to Critically Statically Hoverable (CSH) platforms as any platform that can statically hover, while achieving one of the actuation properties UDT to OD as long as it can actuate its propellers, while it cannot statically hover with its propellers fixed.
While this condition can occur with $N\geq 4$, it can easily be avoided \rev{(for example, the classical quadrotor has $N=4$ and can hover without any propeller tilting)}. On the other hand, any MRAV with $N \leq 3$, and as demonstrated in proposition \ref{prop:hover_n_3}, cannot hover unless its propellers are actuated. As such, and following the previous propositions, we can state the following Lemma:

\begin{lem}\label{lem:CSH}
Any multirotor aerial vehicle, with $N \leq 3$ that can achieve static hovering is a Critically Statically Hoverable platform.
\end{lem}

In what follows we will analyze some example CSH platforms with $N \leq 3$, to demonstrate how such platforms can achieve actuation properties UDT to OD; note that in the following examples we will focus on the minimal configuration that can achieve each corresponding actuation property.

\subsection{UDT CSH platforms}

This class describes platforms that can achieve static hovering while applying forces along one direction. This class is similar in actuation properties to the classical quadrotor platform.
This class can be achieved with $N =3$ and $N = 2$, where the former requires one of the propellers to be actuated along a single orientation, while the later requires both propellers to be actuated, each along a single orientation.
With such platforms, $\vect{A}' \in \nR{6 \times 4}$; as such, the basis of the null space of the moment set $\vect{B}_m \subset \nR{4 \times 1}$, and the hovering direction can be easily calculated as $\zB^0 = \vect{A}_f' \vect{B}_m $.

An example platform with $N = 3$ can be found in \cite{salazar2005stabilization}, which consists of a trirotor platform, with two fixed propellers generating thrust along $\zB$. The third propeller is placed further away from the first two (named the tail propeller), and can rotate about its radial axis.
It is easy to see that such a platform can hover about a single orientation, such that: 
\begin{align}\nonumber
\rotMatB^h \approx \eye{}\\\label{eq:h_0_N3}
\thrusts_0 \approx [\frac{1}{3} \massB g \; \frac{1}{3}\massB g \; \frac{1}{3}\massB g]^T\\ \nonumber
\ai{3} \approx  0
\end{align} 
Note that the hovering orientation and the corresponding controls are slightly deviated from \eqn\eqref{eq:h_0_N3} to ensure moment cancellation; note that the larger the arm length connecting the platform CoM to the propellers' CoM, the smaller this deviation.\\
Similarly, such platforms with $N=2$ can be found in \cite{donadel2014modeling}, where the platform can only hover about a single orientation, such that: 
\begin{align}\nonumber
\rotMatB^h = \eye{}\\ 
\thrusts_0 = [\frac{1}{2}\massB g \; \frac{1}{2}\massB g ]^T\\\nonumber
\orientP_0 = [0 \; 0]^T
\end{align}

\subsection{MDT CSH platforms}

This class describes platforms that can achieve static hovering while applying forces along a plane.
Such platforms are rarely seen in the literature, however, they can be achieved with $N = 3$ with models similar to \cite{papachristos2013model}, or the model presented by the commercial UAV \emph{Cerberus RC Tilt-Rotor}\footnote{http://skybornetech.com/}.
The latter platform is a trirotor, T-shaped, with two frontal propellers being actively independently tilted about their radial axes.

While no design in the literature demonstrated any MDT CSH platform with $N = 2$, such a platform can be achieved, for example, with a birotor where one propeller is actuated radially, while the other is actuated both radially and tangentially.

\subsection{FA and OD CSH platforms}


In this subsection we discuss both FA and OD CSH platforms. First let us note that to achieve FA a platform with $N=3$ ($N = 2$) requires at least $3$ independent propeller actuation ($4$ independent propeller actuation). Correspondingly, $\vect{H}$ should have at least $6$ control inputs.

\begin{prop}\label{prop:CSH_OD_realizability}
A FA CSH with $DoF = 6$ and with $N=3$ or $N=2$ can achieve OD, if each actuation angle does not have a limit, \ie
\begin{align}
\forall (\ai{i}, \betai{i}) \in \orientP \;|\; \ai{i} \in \{-\pi \; \pi\}\; and \; \betai{i} \in \{-\pi \; \pi\}
\end{align}
even if all propellers can only generate uni-directional thrust.
\end{prop}

Note that the above condition is sufficient but not necessary, \ie  a FA CSH with $DoF = 6$ can still achieve OD while its actuation angles are limited.
Also note that not any FA CSH with $DoF = 6$ can achieve OD. However, proposition \ref{prop:CSH_OD_realizability} contrasts with the condition from paper \cite{tognon2017}, where authors in \cite{tognon2017} show that for a platform with fixed propellers, $DoF=7$ is the minimum requirement for a platform to achieve OD.

\begin{proof}
As a proof, let us look at an example platform with $N=3$ and $DoF = 6$, where each propeller can rotate about its radial axis. The corresponding force set $\forceSet$ is shown in figure \ref{fig:force:tri_a123}.
It can be seen from this figure that if the maximum thrust $u_{max}$ is such that $u_{max} \geq \frac{\massB g}{ \sqrt{3}} $, the platform can achieve omnidirectional flight with $DoF = 6$.
\end{proof}

\subsection{Weight Concern}

While adding servomotors to the platform allows the increase of its $DoF$, and allow a platform with $N = 2$ or $3$ to achieve UDT, MDT, FA and OD, we have to note that adding such servomotors adds additional weight to the platform. 
The added weight comes from the servomotor weight, the platform design that allows the incorporation of the servomotor, and any additional mechanism necessary to transfer the rotation motion from the servomotor to the thruster.

In contrast with adding extra thrusters, usually each thruster adds additional lift to the platform. Consequently, while adding thrusters to the platform adds weight, it can also increase the platform's total lift. On the other hand, adding servomotors allows an increase in the platform's $DoF$s without contributing to the platform's lift capacity.
\begin{figure}[tbh]
\centering
\begin{subfigure}[b]{0.25\columnwidth}
	\centering
	\includegraphics[trim=150 0 180 0, clip,width=1\columnwidth]{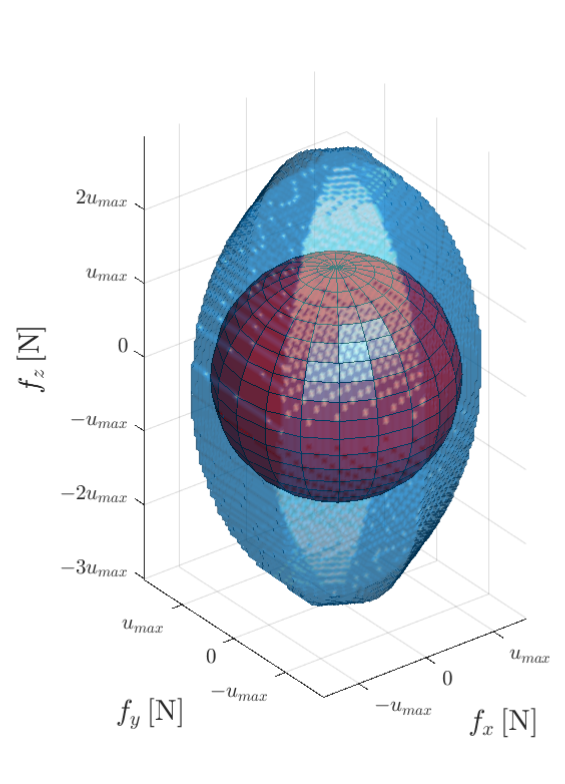}
	\caption{}
	\label{fig:force:tri_a123}
\end{subfigure}
\hfill
\begin{subfigure}[b]{0.3\columnwidth}
	\centering
	\includegraphics[trim=100 0 180 0, clip,width=1\columnwidth]{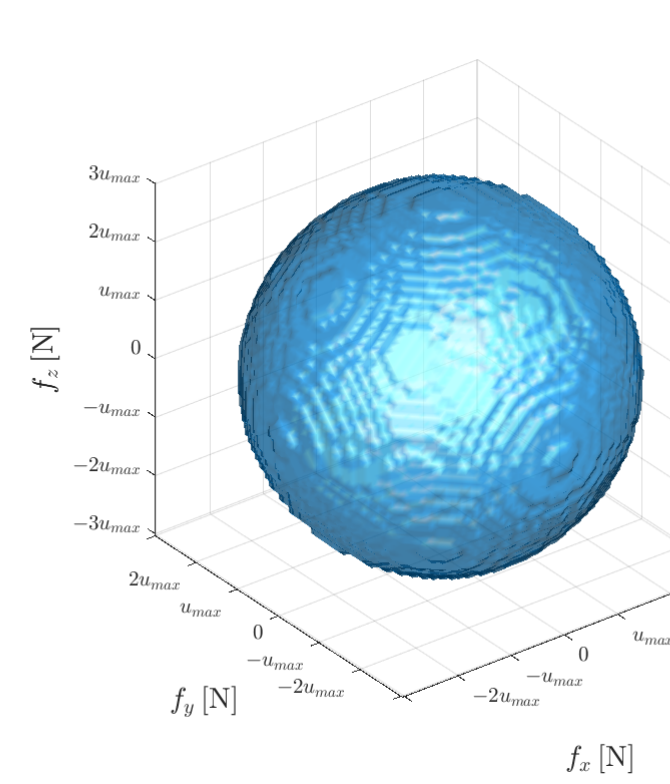}
	\caption{}
		\label{fig:force:tri_a123_b123}
\end{subfigure}
\hfill
\begin{subfigure}[b]{0.4\columnwidth}
	\centering
	\includegraphics[trim=40 20 120 0, clip,width=1\columnwidth]{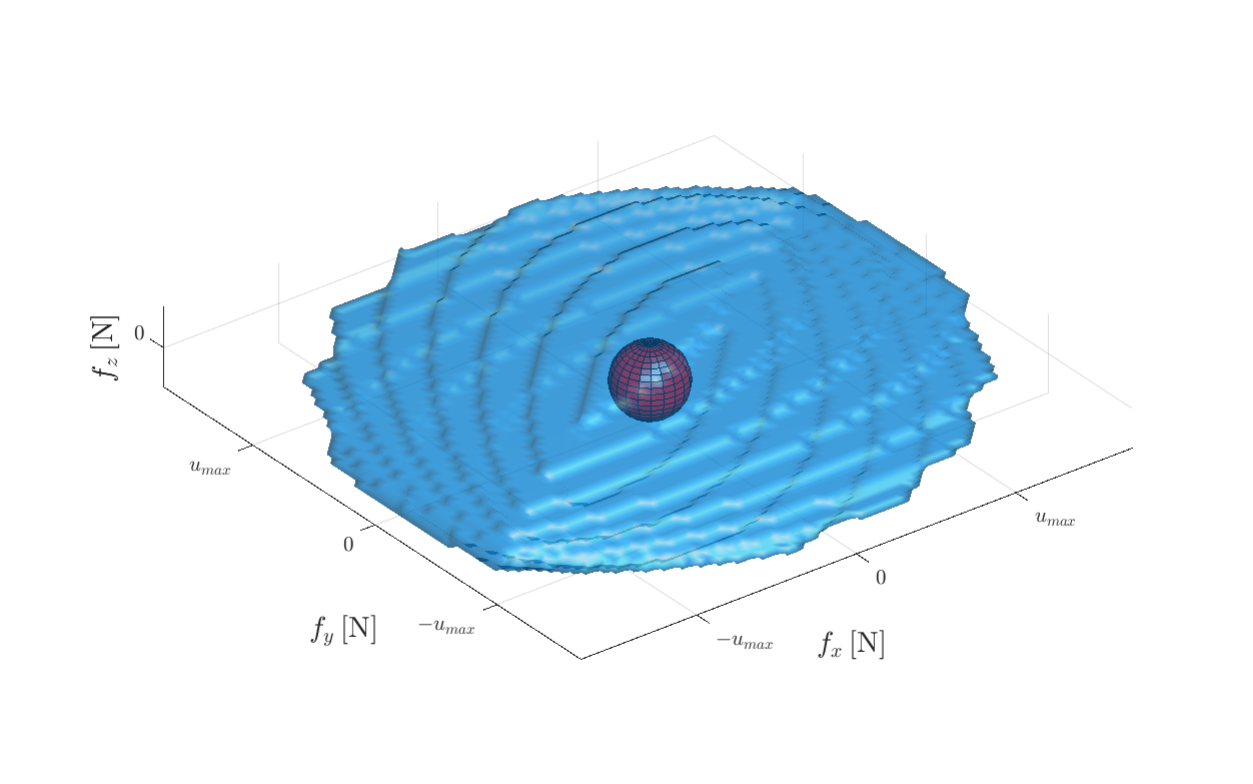}
	\caption{}
	\label{fig:force:tri_a12_b12}
\end{subfigure}
\caption{Force set $\forceSet$ and ODL sphere of a trirotor platform with tiltable propellers, where:\\
a) each propeller can be tilted independently about its radial axis; $N = 3$, DoF$=6$; ODL $\approx u_{max} \sqrt{3} $;\\
b) each propeller can be actuated independently about its radial and tangential axis; $N = 3$, DoF$=9$; ODL $= 3 u_{max}$;\\
c) each propeller can be actuated independently about its radial and tangential axis, with a failed third propeller; $N = 2$, DoF$=6$; ODL $= 0.0246 u_{max}$.
}
\end{figure}
\section{Case Study}\label{sec:case_study}

In this section we will demonstrate the design of a CSH OD platform consisting of $N =3$ propellers, each actuated with two servomotors (\ie $DoF = 9$). We will then show the fail-safe robustness of this platform to the failure of one of its propellers.
Then we will study its ability to sustain hovering at different orientations, and demonstrate the effect of the LHI on the platform's dynamics.

\begin{figure}[tbh]
\centering
	\centering
	\includegraphics[trim=0 0 0 0, clip,width=0.80\columnwidth]{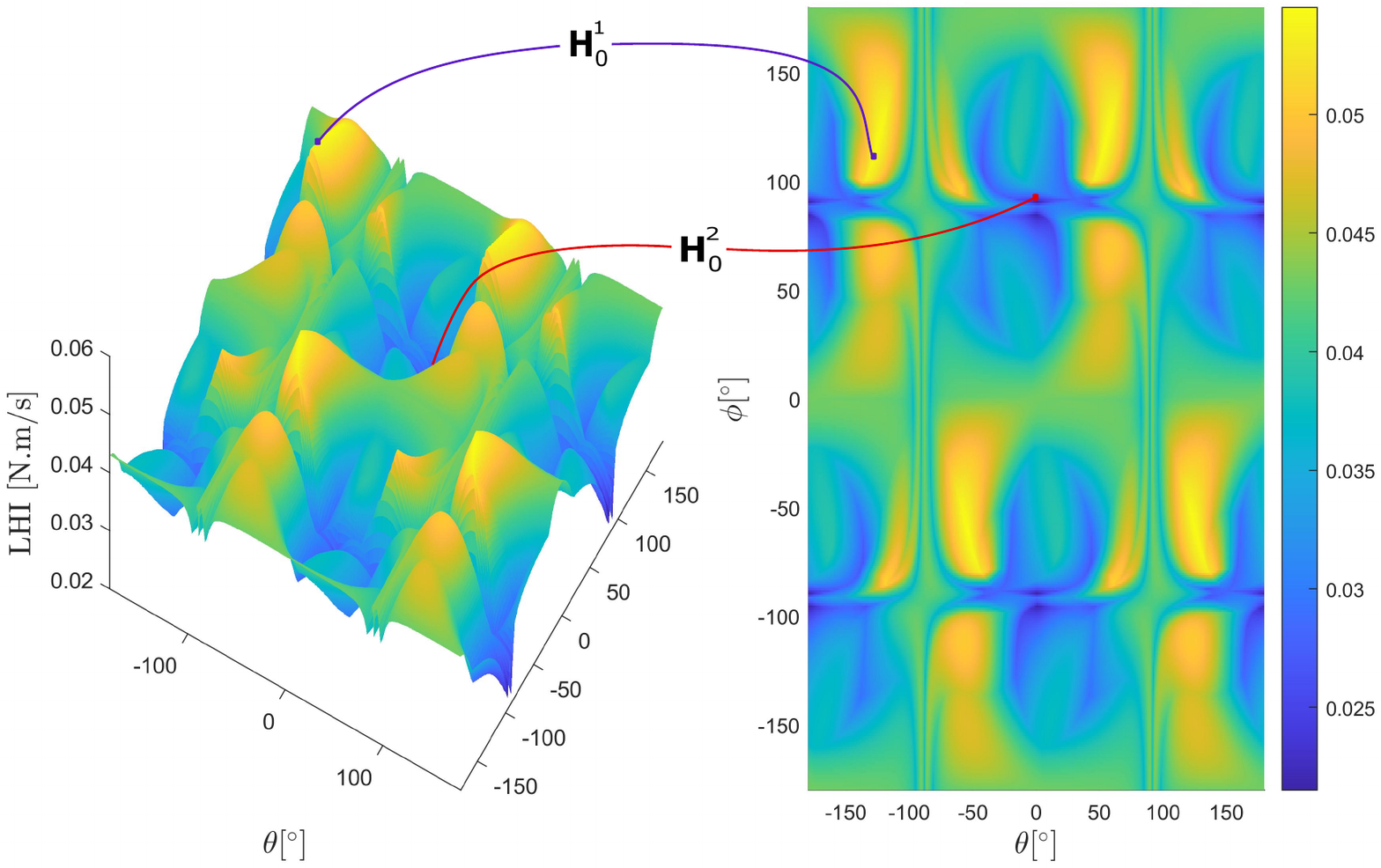}
	\caption{}	
	\label{fig:LHI}
	\caption{LHI of a dual-tilt trirotor platform while hovering at different orientations. Each orientation is parametrized by the platform's rotation about $\xB$: $\phi$ and $\yB$: $\theta$.\\
	Left) perspective view of the LHI metric at different orientation;\\
    right) level plot of the same LHI.	
Two hovering points are highlighted in this figure:\\
$\vect{H}_0^1$: $\theta = -130^\circ$, $\phi = 110^\circ$, LHI $=0.0539 \;[\rm N.m/s]$; \\
$\vect{H}_0^2$: $\theta = 0^\circ$, $\phi = 90^\circ$, LHI $=0.0215 \;[\rm N.m/s]$.\\
This figure should be viewed in color, preferably on a computer screen.
}
\end{figure}
\begin{figure}[tbh]
\centering
\begin{subfigure}[b]{0.49\columnwidth}
	\centering
	\includegraphics[trim=250 0 250 0, clip,width=1\columnwidth]{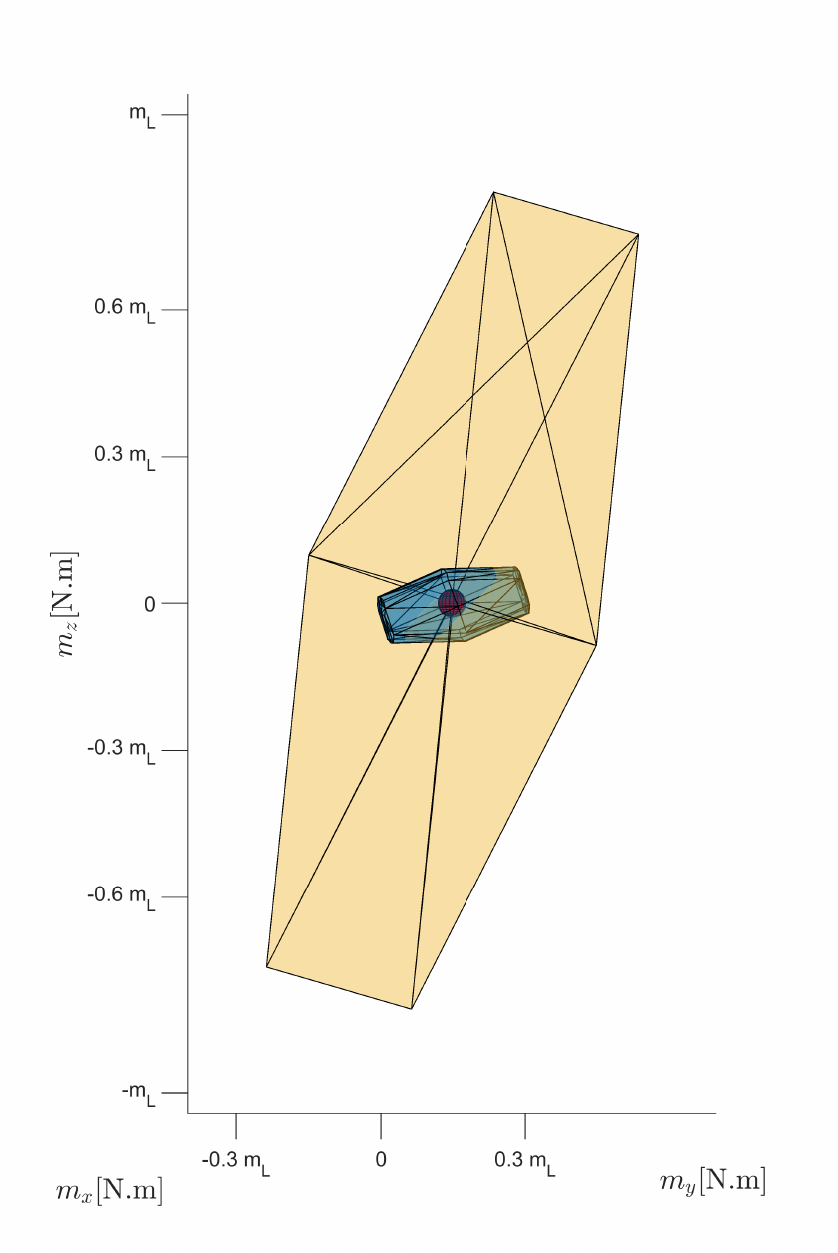}
	\caption{$\vect{H}_0^1$}	
	\label{fig:moment_LHI_2}	
	\end{subfigure}
	\hfill
	\begin{subfigure}[b]{0.49\columnwidth}
	\centering
	\includegraphics[trim=250 0 250 0, clip,width=1\columnwidth]{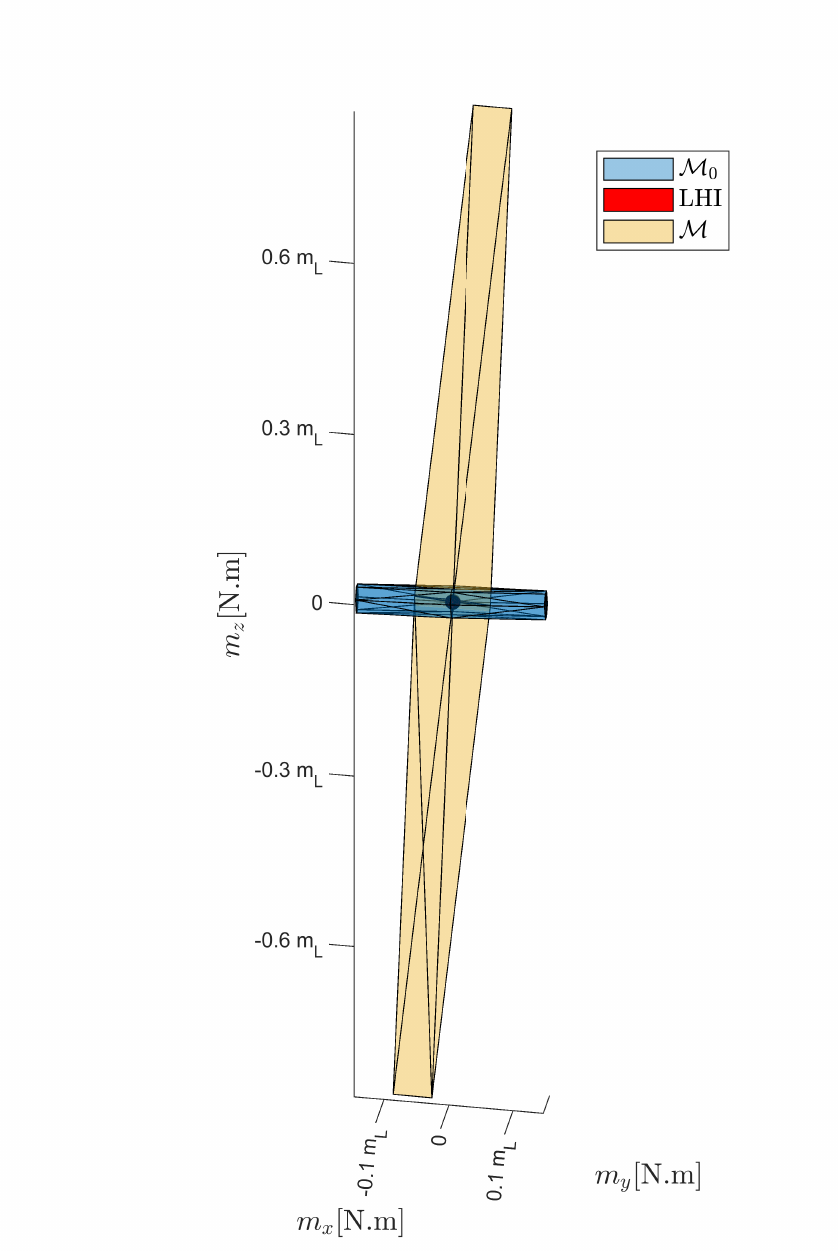}
	\caption{$\vect{H}_0^2$}	
	\label{fig:moment_LHI_2}
	\end{subfigure}
	\caption{For $\vect{H}_0^1$ and $\vect{H}_0^2$: Moment set $\momentSet$ with fixed propeller orientation, local moment set $\momentSetLocal$ and the corresponding inscribed sphere with radius LHI.\\
In this figure $m_l = \massB l$.\\
This figure should be viewed in color, preferably on a computer screen.
}
	\label{fig:moment_LHI}
\end{figure}

\subsection{Design Specifications}
As stated above, let us assume the platform to be actuated with $N =3$ thrusters, where each thruster has a lift coefficient $c_f$ and a drag coefficient $c_\tau$. For the design purpose, let us assume $c_f$ and $c_\tau$ to be variables, and their ratio $r = \frac{c_\tau}{c_f}$ to be constant.
Let us also assume all propellers to be coplanar, \ie all propellers' CoMs are in a plane with the platform's CoM. 
Finally, let us assume all propellers to be equally spaced about the platform's CoM, with the distance between the platform's CoM and the CoM of each corresponding propeller to be denoted by $l$.
As stated above, each propeller is assumed to be actuated with $2$ servomotors, controlling independently its radial and tangential axes. For convenience, let us assume that the servomotors do not have any angle limits. 
As such, the platform's vectoring allocation matrix can be expressed similar to \eqn\eqref{eq:allocation_all}, where

\begin{align}
r_1 = -r_2 = r_3 = r \\
\pPB{1} = [l \; 0 \; 0]^T\\
\pPB{2} = [-\frac{l}{2} \; \frac{l\sqrt{3}}{2} \; 0]^T\\
\pPB{3} = [-\frac{l}{2} \; -\frac{l\sqrt{3}}{2} \; 0]^T
\end{align}
For convenience, and to be able to be able to compute numerically the platform's set of feasible forces $\forceSet$, set of feasible moments $\momentSet$, and the platform's LHI, let us assume $l = 0.2 \;[\rm m]$ and $r = 0.012 \;[\rm m]$.
Finally, let us denote by $\vect{A}_{1\rightarrow 3}$ as the $\vect{A}$ matrix in the case where all three propellers are functional, while $\vect{A}_{1\rightarrow 2}$ as the corresponding matrix in the case where only the first two propellers are functional.

\subsection{Static Hovering Realization}

First of all, it is easy to see that the platform is fully actuated (FA) in the case where all its propellers are functional and in the case where only two of its propellers are functional as follows:

\begin{align}
\rank(\vect{A}_{1\rightarrow 3, f}) = 3 \; & \; \rank(\vect{A}_{1\rightarrow 3}) = 6 \\
\rank(\vect{A}_{1\rightarrow 2, f}) = 3 \; & \; \rank(\vect{A}_{1\rightarrow 2}) = 6 
\end{align}

Then, to understand if the platform can hover in all orientations let us plot the platform's force set $\forceSet$: figure~\ref{fig:force:tri_a123_b123} shows $\forceSet$ for the platform with all propellers functional, while figure~\ref{fig:force:tri_a12_b12} shows $\forceSet$ of the platform with one propeller failure. \rev{Note that the plots in these figures were constructed using the method described in \cite{hamandi2021ijrr}.
Also no}te that these two plots are shown with respect to the maximum thrust $u_{max}$ each propeller can produce, where from \eqn\eqref{eq:force_set} we can see that the corresponding set of feasible forces is linearly dependent on $u_{max}$ if all propellers are identical.
While in the case all propellers are functional $\forceSet$ is a sphere with radius equal to the ODL, figure~\ref{fig:force:tri_a12_b12} shows the ODL of the platform with a failed propeller. We can see from the two plots that the platform can achieve OD flight if $\massB \leq 3 u_{max}$ if all three propellers are functional, while in the case where only two of its propellers are functional, the required mass decreases to $\massB \leq 0.0246 u_{max}$.
Following the data from the two plots, if $\massB$ and the platform's thrusters are chosen such that $\massB \leq 0.0246 u_{max}$, the platform should be able to achieve OD flight with all its propellers functional, and in the case one of its propellers fails.
Note that in the case of a failed propeller, the weight to maximum thrust requirement is difficult to achieve with off the shelf components. As such, it is usually preferable to design the platform to be able to achieve static hovering at any orientation after propeller failure. Note from figure~\ref{fig:force:tri_a12_b12} that the corresponding platform could still hover with $\zB$ aligned with $\xW$ or $\yW$, and with a weight to thrust ratio $\massB \leq 1.48 u_{max}$.

\subsection{Static Hovering Sustainability}

While the above platform is OD, the platform does not have the same hovering ability $\forall \rotMatB^h \in \mathcal{R}$. To understand the effect of the hovering orientation $\rotMatB^h$ on the platform's hovering ability, we first plot the LHI of the platform at different orientations in figure~\ref{fig:LHI}.
This figure shows that the platform's LHI ranges~$\in \{0.0215 \; 0.0545 \} [\rm N.m/s]$ .

To highlight Critical Static Hovering property of the platform, we show in figure~\ref{fig:moment_LHI} for two distinct attitudes: \begin{inparaenum}
\item the corresponding set of feasible moments if all propellers are held stationary at their corresponding orientation $\momentSet$; 
\item and the corresponding set of feasible moment changes at the current orientation $\momentSetLocal$.
\end{inparaenum}
Note that to be able to compute numerically the corresponding convex hulls, we assumed $\max \thrustsDer = 200 \;[\rm Hz/s]$ while $\max \orientPD = 4.1 \;[\rm rad/s]$, where these assumptions match with a number of thrusters and servomotors that we surveyed.

First of all, from figure~\ref{fig:moment_LHI} we can clearly see that, \rev{if all propellers are held stationary}, the set of feasible moments at hover is a two dimensional plane, and as such, the platform can be considered a CSH platform.
From the same figure, we can also see that $\vect{0} \in \interior(\momentSetLocal)$ in both cases, although the corresponding LHI changes in diameter.
As such, the platform is able to apply moments at the corresponding hovering orientation by changing the propellers' orientations and the applied thrust.
In addition, as the LHI at the two hovering points $\vect{H}_0^1$ and $\vect{H}_0^2$ changes drastically, it is expected that the platform's dynamic response will change depending on the corresponding platform's hovering attitude.

\begin{figure}[tbh]
\centering
\includegraphics[trim=0 0 0 0, clip,width=0.5\columnwidth]{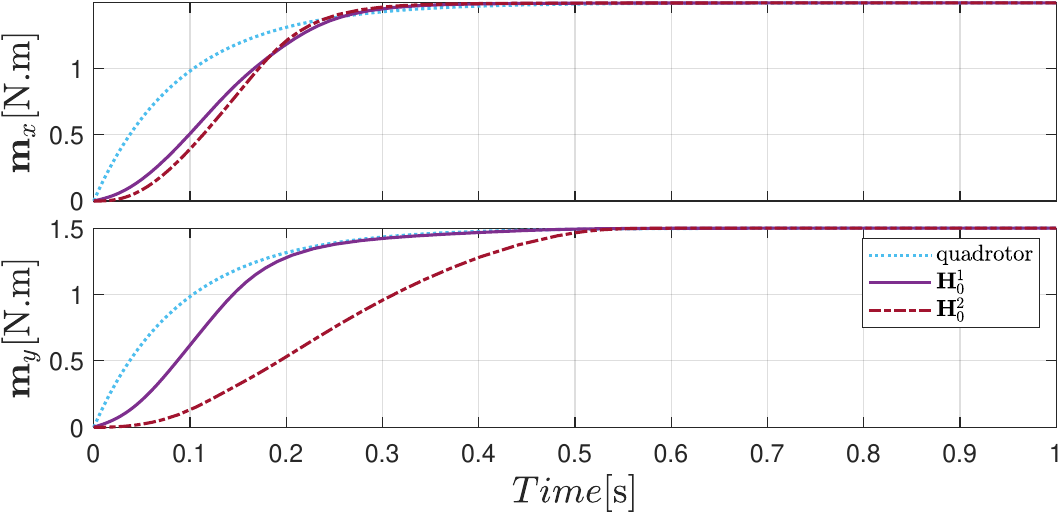}
\caption{Moment step response dynamics of the dual-tilt trirotor at the two hovering points $\vect{H}_0^1$ and  $\vect{H}_0^2$, while applying\\
(top) $\totalMoment_x = 1.5\;[\rm N.m]$,\\
(bottom) $\totalMoment_y = 1.5\;[\rm N.m]$.}\label{fig:moment_tracking}
\end{figure}

To understand this phenomena, \rev{we study in what follows the ability of a platform to apply a desired moment, or to change the orientation of its applied force. In each of the simulations, the platform is applying a control input $\vect{H}_0$ that allows static hovering at the desired orientation, while applying zero moment.} Figure~\ref{fig:moment_tracking} shows the step response of the platform when trying to achieve a desired moment along $\xB$ and $\yB$. The plot also shows the behavior of a classical coplanar colinear quadrotor while trying to apply similar moments.
We can see from figure~\ref{fig:moment_tracking} that while the platform is an OD platform, and thus can hover in all orientations, its ability to modify its moment along $\yB$ is largely decreased while hovering at $\vect{H}_0^2$ as compared to $\vect{H}_0^1$. Moreover, in both plots, the applied moment at $\vect{H}_0^2$ in the first $50 \; ms$ is negligible as compared to the one applied by the quadrotor. This is due to the large change in propeller orientation required to applied the corresponding moment.
While such propeller orientation changes are less visible in the case of $\vect{H}_0^1$, we can see that the quadrotor, in both cases can apply moment faster than the dual-tilt platform.

Finally, it is noted that the platform's LHI also reflects its ability to change the force orientation at different hovering orientations. Figure~\ref{fig:force_tracking} shows the force orientation tracking dynamics of the dual-tilt platform at $\vect{H}_0^1$ and $\vect{H}_0^2$. In figure~\ref{fig:force_tracking}, the platform's force direction is modified such that the desired force direction is rotated to a desired direction as follows:
\begin{align}
\vect{d}^d = \rotMat_x(5^\circ) \vect{d}(\vect{H}_0)
\end{align}

\begin{figure}[tbh]
\centering
\includegraphics[trim=0 0 0 0, clip,width=0.5\columnwidth]{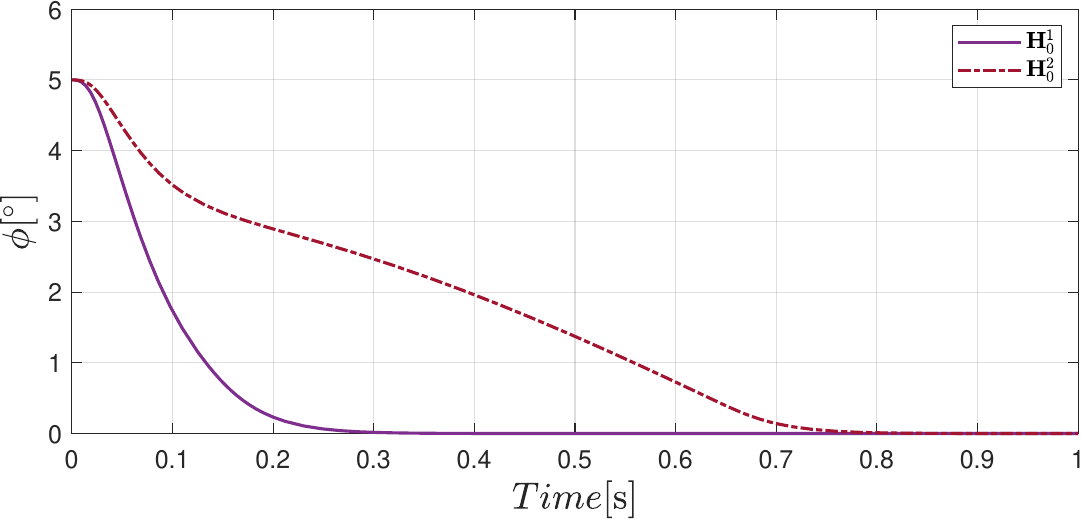}
\caption{Force orientation response dynamics of the dual-tilt trirotor at the two hovering points $\vect{H}_0^1$ and  $\vect{H}_0^2$. This figure shows the angular error between the force orientation $\vect{d}$ at different time steps and the desired one $\vect{d}^d$.}\label{fig:force_tracking}
\end{figure}

Figure~\ref{fig:force_tracking} shows clearly the ability of the dual-tilt trirotor to apply the desired force orientation change in approximately $0.2 \;[\rm s]$ at $\vect{H}_0^1$, while requiring $0.7 \;[\rm s]$ at $\vect{H}_0^2$.
\section{Conclusions}\label{sec:conc}

In this paper we present a thorough analysis of the static hovering ability of MRAV platforms with tiltable propellers, and present the algebraic conditions to achieve and sustain desired static hovering properties. To do so, we formally introduce and model such platforms, and present an allocation map for any MRAV with tiltable propellers. Then we present the conditions for a platform to hover at a single orientation, in a plane, or in $\nR{3}$. Furthermore, we present the conditions for a platform to sustain achieved hovering, and present a numerical metric that reflects the stability of each hovering orientation. Based on the presented conditions, we introduce a class of MRAVs referred to as Critically Statically Hoverable (CSH) platforms, which are platforms that cannot sustain static hovering without active propeller tilting. In addition, we present a thorough simulation study to test and validate the presented conditions. These simulations show clearly the effect of the presented numerical metric (LHI) on the platform's dynamic response and ability to apply different forces and moments.
\rev{Moreover, these simulations show clearly that a platform with $N\leq3$ is a CSH platform, where the dimension of the moment set with fixed propellers is less than $N$, even if the platform is OD when all of its propellers are fully functional.}

Based on the fundamental conditions and properties discussed in this paper, future studies could include the optimization of minimal CSH platform design to maximize the LHI metric at different orientations. Alternatively, future studies could use the LHI metric in the control loop, where a controller that advantages orientations with higher LHI is expected to achieve more stable static hovering.

Moreover, in the case of MRAV with $N \geq 4$, CSH conditions could occur after propeller failure. Accordingly, it would be interesting to extend the conditions from this work to understand the fail-safe robustness of generic MRAVs with tiltable propellers.

\section*{Acknowledgment}
The authors acknowledge the support provided by the Khalifa University of Science and Technology Award No. RC1-2018-KUCARS and Project CIRA-2020-082 .
\bibliography{./Bib}

\end{document}